\documentclass{article}

 \usepackage[preprint]{neurips_2026}

\usepackage[utf8]{inputenc} % allow utf-8 input
\usepackage[T1]{fontenc}    % use 8-bit T1 fonts
\usepackage{hyperref}       % hyperlinks
\usepackage{url}            % simple URL typesetting
\usepackage{booktabs}       % professional-quality tables
\usepackage{amsfonts}       % blackboard math symbols
\usepackage{nicefrac}       % compact symbols for 1/2, etc.
\usepackage{microtype}      % microtypography
\usepackage{xcolor}         % colors

\usepackage{multirow}
\usepackage{graphicx}
\usepackage{subcaption}
\usepackage[table]{xcolor}
\usepackage{amsmath}
\usepackage{amssymb}
\usepackage{algorithm}
\usepackage{algpseudocode}
\usepackage{xcolor}
\usepackage{array}
\usepackage{colortbl}
\usepackage{amsthm}

\definecolor{taskblue}{RGB}{72,118,255}
\definecolor{taskred}{RGB}{238,76,76}
\definecolor{mergeyellow}{RGB}{255,196,0}
\definecolor{flatbg}{RGB}{232,239,255}
\definecolor{alignbg}{RGB}{255,235,235}
\definecolor{mergebg}{RGB}{255,246,210}

\newtheorem{theorem}{Theorem}[section]  
\newtheorem{corollary}[theorem]{Corollary} 
\newtheorem{lemma}[theorem]{Lemma}

\newtheorem{assumption}[theorem]{Assumption}

\usepackage{wrapfig}
\usepackage{hyperref}
\usepackage{enumitem}

\usepackage{comment}
\title{When Privacy Hurts Mergeability: Geometry-Aware Model Merging under Differential Privacy}

\author{%
  Jin Liu\\
  Xidian University\\
  \texttt{jinliu9787@gmail.com} \\
  \And
  Junkang Liu \\
  Tianjin University \\
  \texttt{junkangliukk@gmail.com} \\
  \And
  Ning Xi \\
  Xidian University \\
  \texttt{nxi@xidian.edu.cn} \\
  \And
  Yinbin Miao \\
  Xidian University \\
  \texttt{ybmiao@xidian.edu.cn} \\
  \And
  Dawei Wei \\
  Xidian University \\
  \texttt{weidawei58@gmail.com} \\
  \And
  Ke Cheng \\
  Xidian University \\
  \texttt{chengke@xidian.edu.cn} \\
  \And
  Jianfeng Ma \\
  Xidian University \\
  \texttt{jfma@mail.xidian.edu.cn} \\
}

\begin{document}

\maketitle

\begin{abstract}
Model merging promises to construct a single multi-task model from independently fine-tuned task models without accessing the original task data. 
This makes it attractive when task data cannot be centralized, but released task models may still leak private fine-tuning data.
Differential privacy (DP) provides a principled mechanism for limiting such leakage, yet its effect on model merging remains poorly understood. 
In this paper, we study the geometry of differentially private model merging and identify two geometric obstacles that make private task models difficult to merge: \emph{local sharpness}, which makes task losses sensitive to the parameter displacement induced by merging, and \emph{reference drift}, which measures the displacement of private task models from the shared pretrained initialization and amplifies cross-task interference. Based on these observations, we propose \textbf{DP-Merging}, a geometry-aware framework that improves the mergeability of differentially private task models. DP-Merging uses a DP-compatible sharpness-aware objective to guide each private task model toward flatter loss regions, and a reference-based alignment regularizer to keep task models close to the shared pretrained initialization.
We derive a merge-gap upper bound showing that reducing local curvature and reference drift tightens the bound on the loss increase induced by merging.
Experiments on vision and language tasks across multiple privacy budgets show that DP-Merging consistently improves private merged-model performance while preserving the privacy guarantees of the underlying DP fine-tuning procedures.

\end{abstract}
\section{Introduction}
\label{sec:1}

Fine-tuning pretrained models has become the standard way to adapt foundation models to diverse downstream tasks~\cite{liu2026rethinking,wang2026taming}. 
As task-specific models accumulate, storing and deploying them independently becomes increasingly costly. 
Multi-task learning can integrate these capabilities into a model, but requires joint access to data from all tasks and expensive retraining~\cite{fifty2021efficiently,agiza2024mtlora}. 
Model merging has recently emerged as a practical alternative~\cite{yang2026model}: it directly combines independently fine-tuned models in parameter space to obtain a unified model, without accessing the original task data. This makes model merging particularly attractive when task data are decentralized or cannot be shared.

Despite its potential, data-free merging is not inherently privacy-preserving. Although raw data are not shared during merging, the released task-specific weights and task vectors are derived from private data, posing a risk of sensitive information leakage~\cite{zhang2024badmerging,wang2025purity,yuan2025merge}. Differential privacy (DP)~\cite{dwork2006calibrating} provides a rigorous framework for limiting such leakage by bounding the influence of any single training example on the released models. A natural solution is therefore to fine-tune each task model under DP and then merge the resulting private models. However, as we show empirically, this straightforward solution can suffer from a degradation that goes beyond the utility loss of individual private models: DP can specifically damage the geometric compatibility required for parameter-space merging. This raises a central question: ~\textit{what makes differentially private models difficult to merge, and how can we improve their mergeability under privacy constraints?}

\begin{figure*}[t]
    \centering
    
    \begin{subfigure}[b]{0.32\textwidth}
        \centering
        \includegraphics[width=\textwidth]{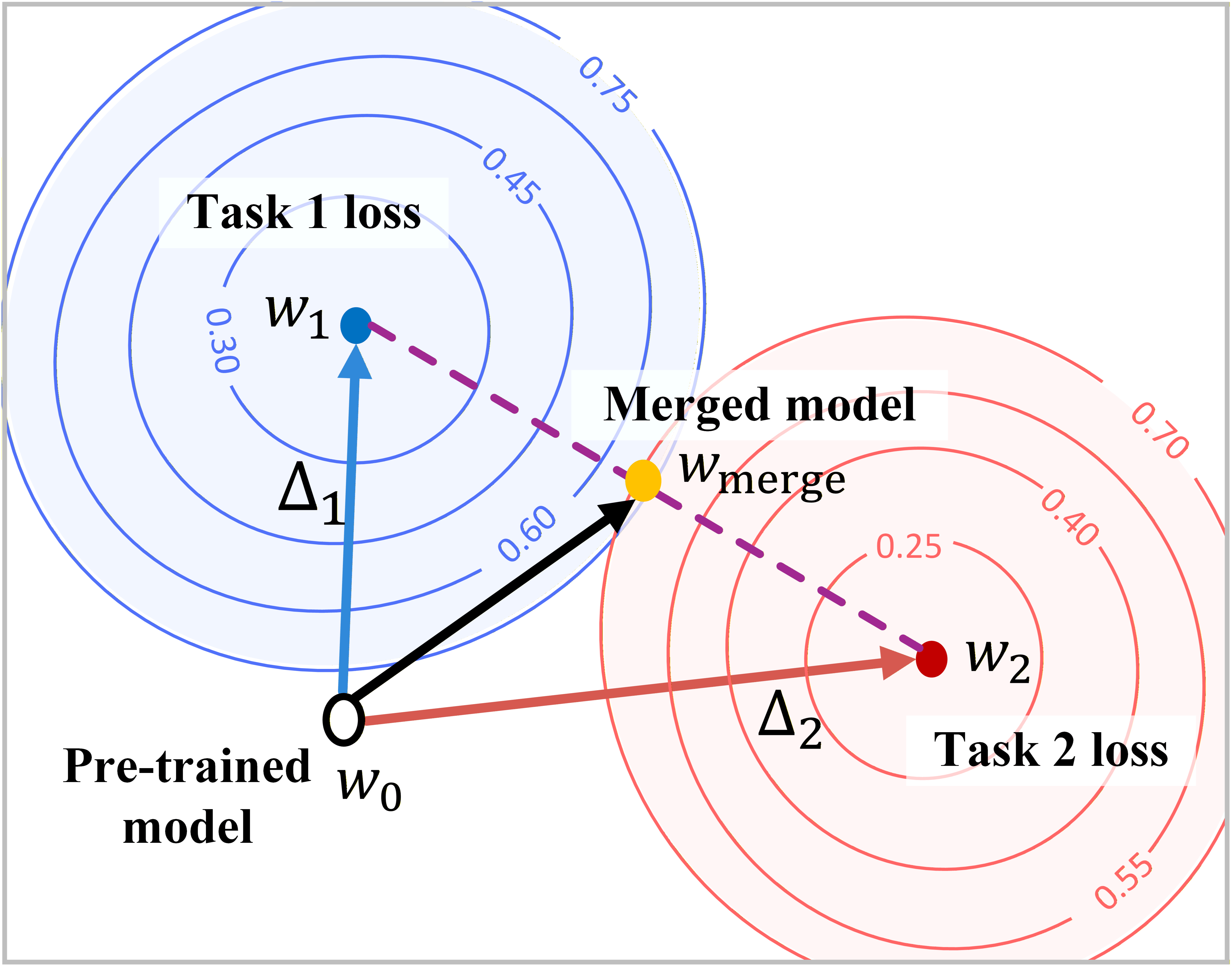}
        \caption{Non-private mergeing}
        \label{fig:barrier_a}
    \end{subfigure}
    \begin{subfigure}[b]{0.32\textwidth}
        \centering
        \includegraphics[width=\textwidth]{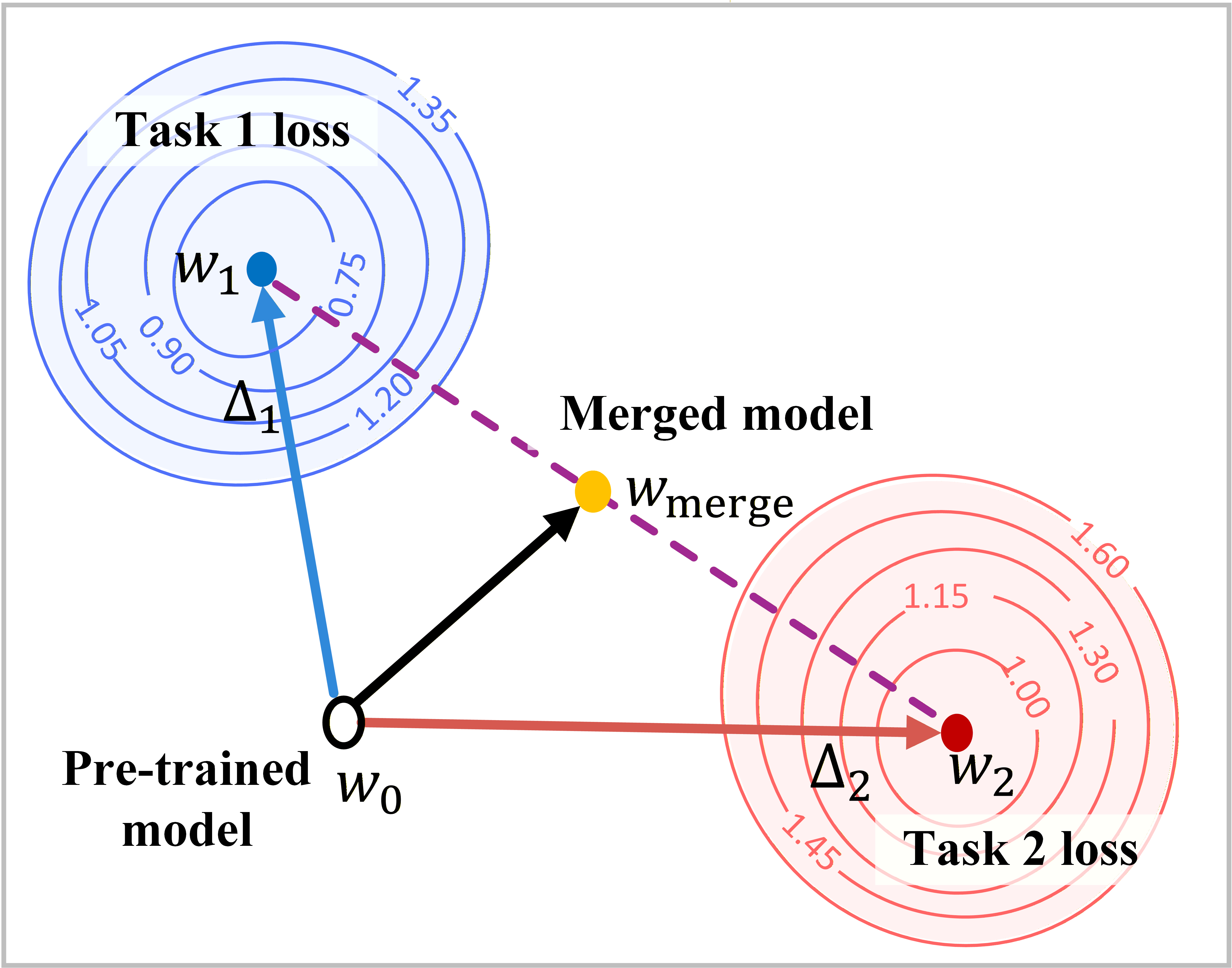}
        \caption{Naive DP merging}
        \label{fig:barrier_b}
    \end{subfigure}
    \begin{subfigure}[b]{0.32\textwidth}
        \centering
        \includegraphics[width=\textwidth]{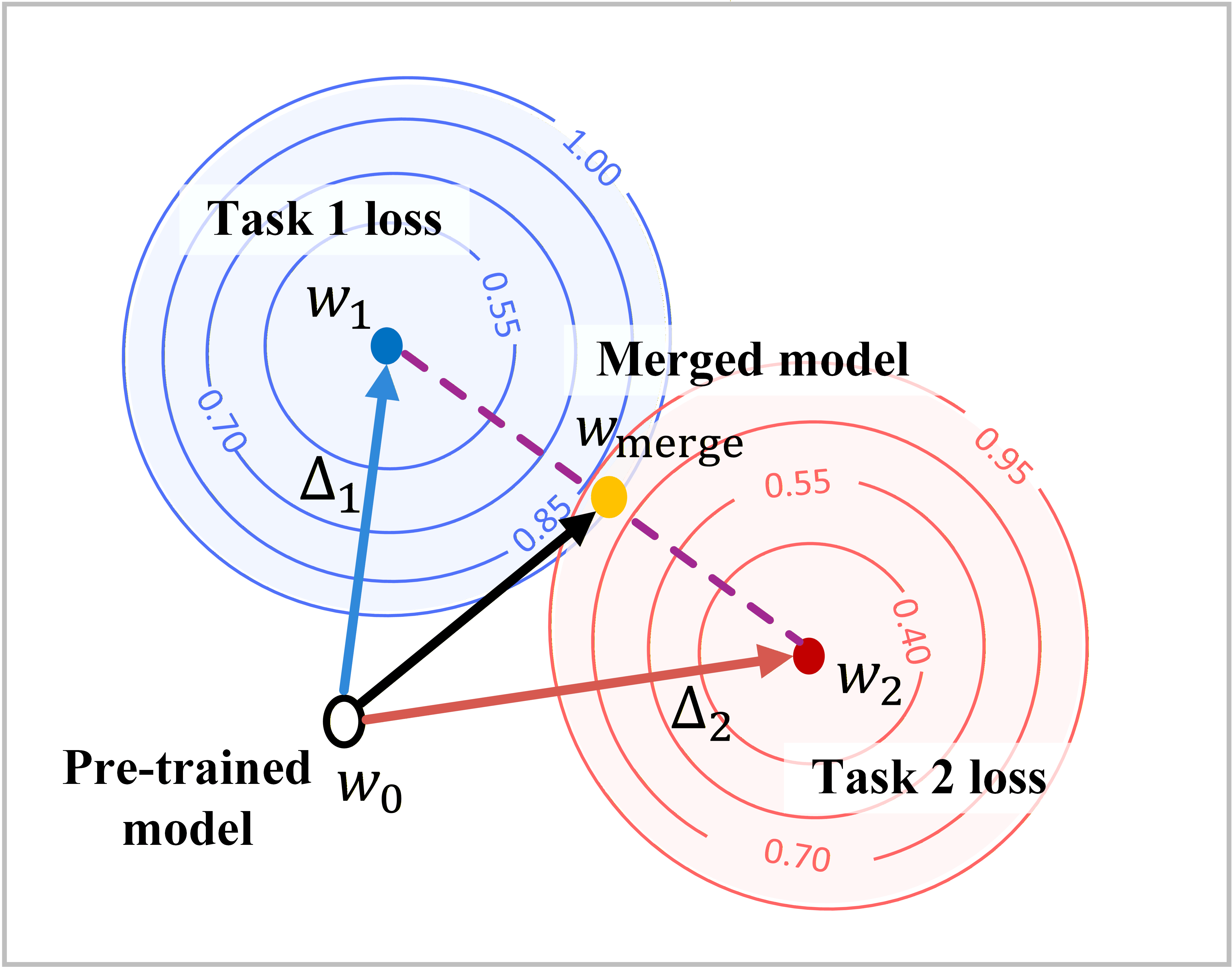}
        \caption{DP-Merging (Ours)}
        \label{fig:barrier_c}
    \end{subfigure}
\caption{
Geometric illustration of model merging under DP.
Task vectors are fine-tuned from $w_0$ and merged in parameter space.
(a) Without DP, flat and compatible solutions merge effectively.
(b) Naive DP can yield sharper solutions with weaker reference alignment, leading to high-loss merging.
(c) DP-Merging improves flatness and reference alignment, yielding lower-loss merging under DP.
}
    \label{fig:geometry}
    \vspace{-2mm}
\end{figure*}

We argue that the difficulty is not merely that DP lowers the accuracy of the models being merged. Rather, parameter-space merging relies on geometric compatibility among task-specific solutions, which can be disrupted by DP fine-tuning. Prior work suggests that merging often benefits when independently fine-tuned models remain geometrically compatible~\cite{ainsworth2023git}. In particular, interpolation or averaging tends to incur smaller loss increase when the path between models stays in relatively flat, low-loss regions~\cite{pena2023re,ito2025linear}. In the private setting, these geometric conditions become more fragile. The clipping and noise used by DP may leave each task solution in a sharper local basin, making its loss more sensitive to the parameter displacement induced by merging. Meanwhile, independent private fine-tuning can move task models farther from the shared pretrained initialization, increasing their mutual mismatch and amplifying cross-task interference. We refer to these two obstacles as \emph{local sharpness} and \emph{reference drift}. These effects make private task models more sensitive to merge-induced displacement and less compatible in parameter space, leading to high-loss merged solutions, as illustrated in Figure~\ref{fig:geometry}.

Motivated by this geometric view, we propose \textbf{DP-Merging}, a simple geometry-aware framework for differentially private model merging. 
Rather than designing a new post-hoc merging operator, DP-Merging applies two minimal interventions during private fine-tuning to restore the geometric conditions required by merging. 
First, a DP-compatible sharpness-aware objective encourages each private task model to lie in a flatter local region, reducing its sensitivity to merge-induced displacement. 
Second, a reference-anchored alignment regularizer keeps task vectors close to the shared pretrained initialization, limiting reference drift and reducing cross-task interference. 
Since the final merging step only processes DP-released task models, it remains a post-processing operation and incurs no additional privacy loss.
We provide a theoretical merge-gap analysis showing that reducing local curvature and reference drift improves mergeability, and empirically validate DP-Merging across vision and language tasks under multiple privacy budgets.

In summary, we make the following contributions:
\begin{itemize}
\item 
We identify a geometric failure mode of differentially private model merging, termed \emph{DP-induced mergeability degradation}. We show that DP fine-tuning can make task-specific models harder to merge by increasing local sharpness and reference drift, which amplify parameter interference after merging.

\item 
We propose \textbf{DP-Merging}, a simple geometry-aware private fine-tuning framework. 
DP-Merging combines a DP-compatible sharpness-aware objective with a reference-based alignment regularizer to produce private task models that are flatter and more geometrically aligned for parameter-space merging.

\item 
We provide theoretical and empirical evidence. 
Our merge-gap bound shows that the loss increase after merging is controlled by local curvature and merge-induced parameter displacement, and experiments on vision and language tasks across multiple privacy budgets demonstrate consistent improvements over standard DP fine-tuning followed by merging.
\end{itemize}

%\vspace{-2mm}
\section{Related Work}
\vspace{-1mm}
\paragraph{Non-private model merging.}
Existing work on model merging mainly studies the non-private setting. Representative approaches include direct weight-space averaging, such as Weight Averaging~\cite{wortsman2022model}, task-vector composition such as Task Arithmetic~\cite{ilharco2023editing}, and more structured rules based on parameter importance or interference resolution, such as Fisher-weighted merging~\cite{matena2022merging}, RegMean~\cite{jin2023dataless}, TIES-Merging~\cite{yadav2023tiesmerging}, PCB Merging~\cite{du2024parameter}, and WUDI-Merging~\cite{cheng2025whoever}. These methods typically assume direct access to task-specific weights, parameter deltas, or importance statistics. 
In contrast, we study a privacy-constrained setting where task-specific models must be obtained under differential privacy before they can be released and merged.

\vspace{-1mm}
\paragraph{Differentially private fine-tuning.}
DP fine-tuning of pretrained models has become increasingly practical, with most methods relying on per-example gradient clipping and Gaussian noise injection to improve the privacy--utility trade-off for a single released model~\cite{yu2022differentially,park2023differentially,bu2024differentially,li2024fine}. Other work also explores alternative private fine-tuning paradigms~\cite{bu2024differentially}, such as forward-pass perturbation~\cite{du2023dp} or zeroth-order optimization~\cite{zhang2024dpzero}. However, these methods are designed mainly to preserve the utility of individual private models. 
They do not address whether multiple privately fine-tuned task models remain geometrically compatible for post-hoc merging, which is the focus of our work.

\vspace{-1mm}
\paragraph{Geometry of mergeability and sharpness.}
Prior work indicates that the success of model merging can be influenced not only by the merging method but also by geometric properties of the models, such as alignment in parameter space, low-loss connectivity, and landscape structure.
Averaging or interpolation is more reliable when fine-tuned models lie in compatible low-loss regions or admit low-loss connectivity~\cite{garipov2018loss,draxler2018essentially,wortsman2022model}, while direct merging can fail when parameters are misaligned, for example due to permutation symmetries~\cite{ainsworth2023git,ito2025linear,zhang2025beyond}. Sharpness-aware optimization methods, such as Sharpness
Aware Minimization (SAM~\cite{foret2021sharpnessaware}) seeks flatter solutions by optimizing losses under local parameter perturbations, and is closely related to interpolation stability and low-loss connectivity~\cite{lee2025mitigating}. These studies clarify important geometric conditions for merging in non-private settings, but do not examine how privatization perturbs those conditions before composition.
\section{Preliminaries}
\label{sec:preliminaries}

\begin{figure*}[h]
    \centering
    \begin{subfigure}[b]{0.24\textwidth}
        \centering
        \includegraphics[width=\textwidth]{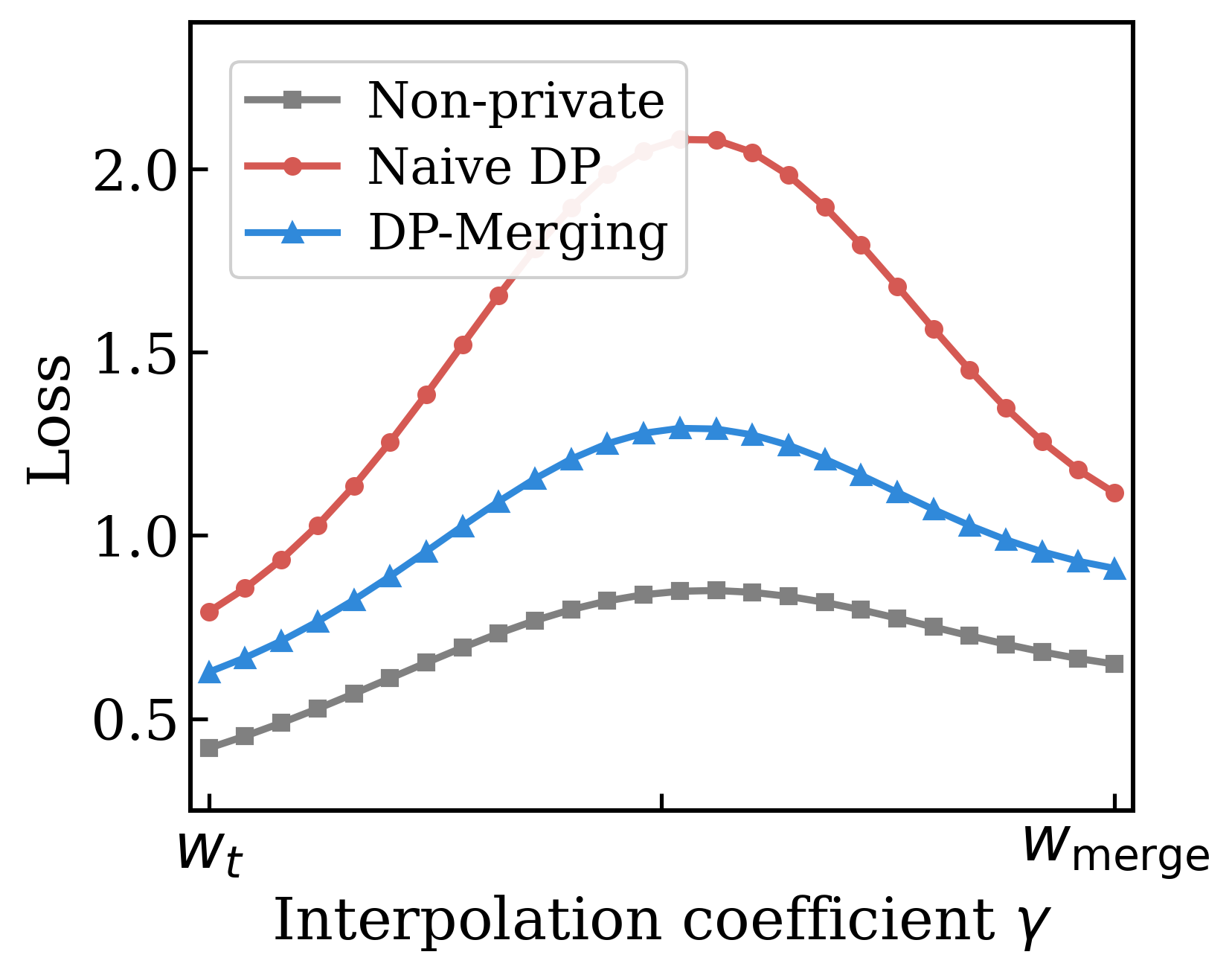}
        \caption{EuroSAT}
        \label{fig:bar_1}
    \end{subfigure}
    \begin{subfigure}[b]{0.24\textwidth}
        \centering
        \includegraphics[width=\textwidth]{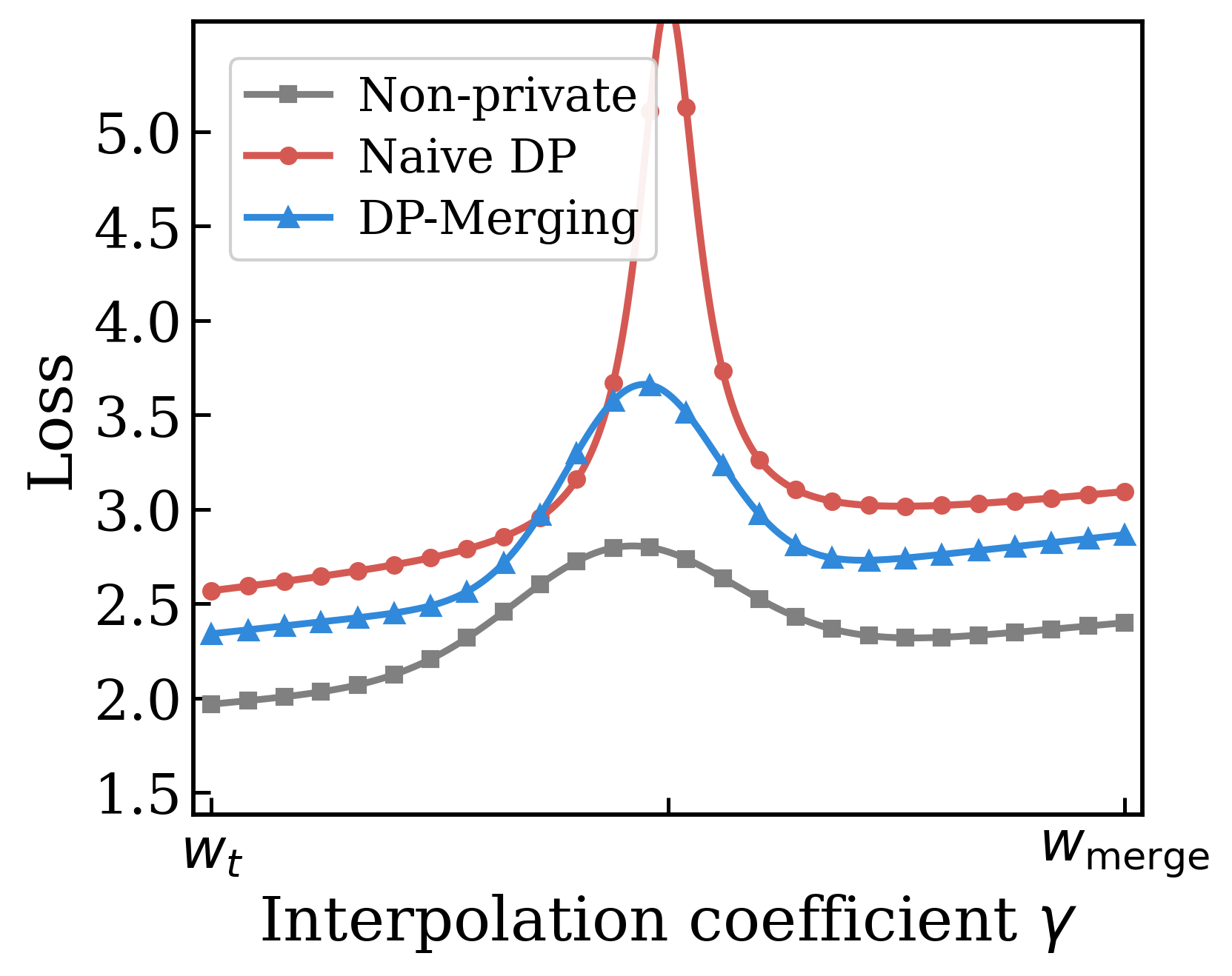}
        \caption{SUN397}
        \label{fig:bar_2}
    \end{subfigure}
    \begin{subfigure}[b]{0.24\textwidth}
        \centering
        \includegraphics[width=\textwidth]{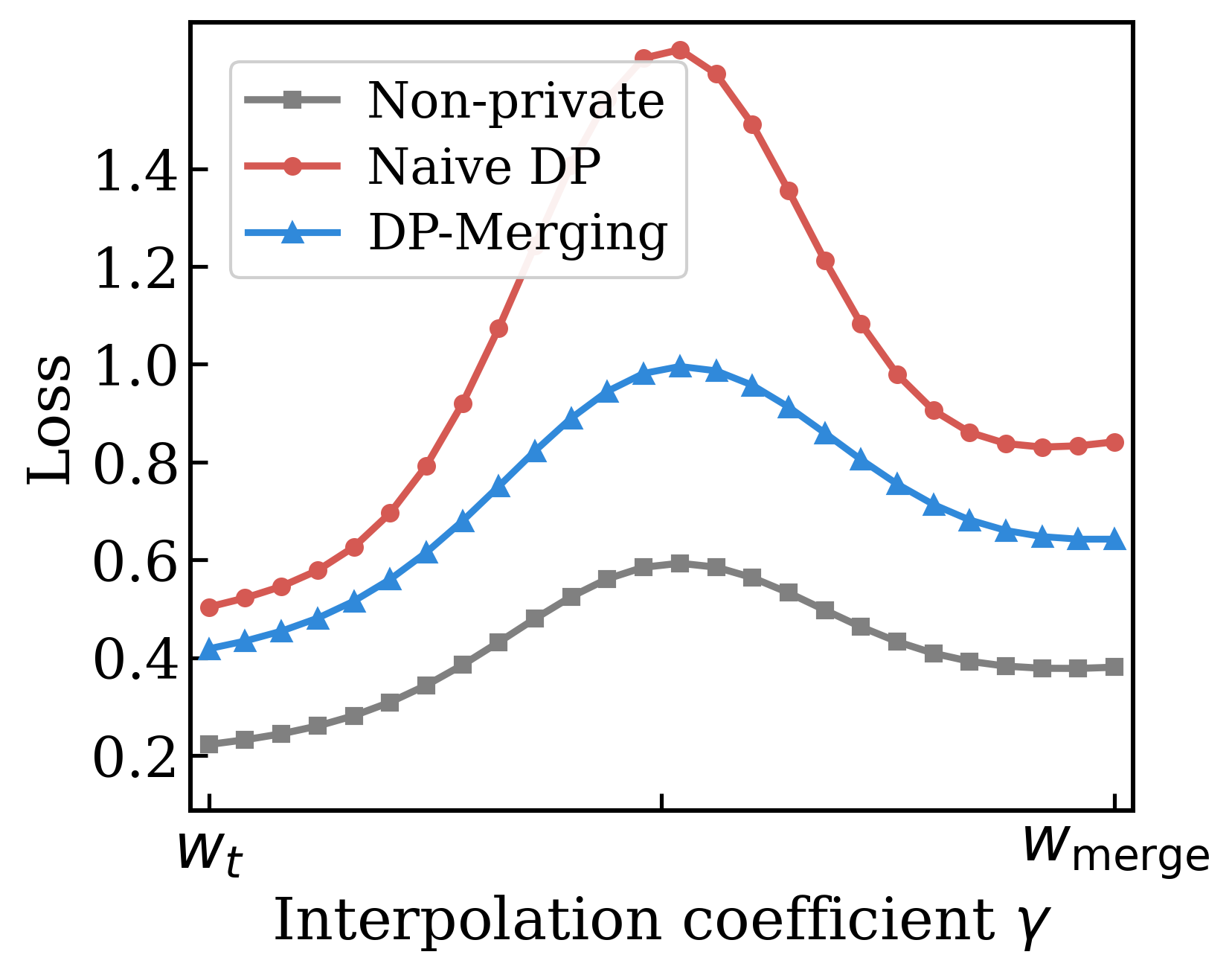}
        \caption{SST-2}
        \label{fig:bar_3}
    \end{subfigure}
        \begin{subfigure}[b]{0.24\textwidth}
        \centering
        \includegraphics[width=\textwidth]{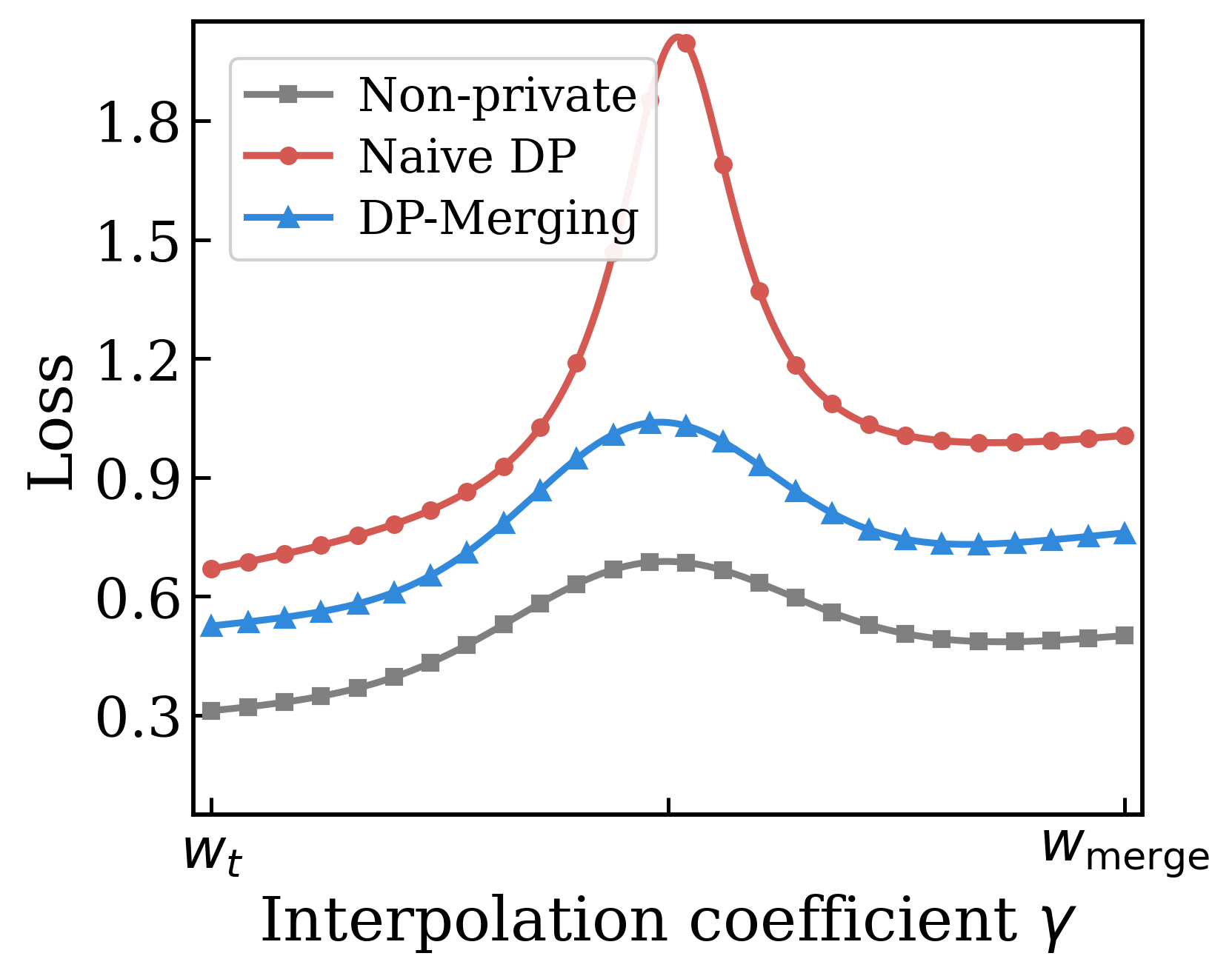}
        \caption{QNLI}
        \label{fig:bar_4}
    \end{subfigure}
\caption{
Interpolation loss landscapes from $w_t$ to $w_{\mathrm{merge}}$.
We evaluate test loss along $w(\gamma)=(1-\gamma)w_t+\gamma w_{\mathrm{merge}}$.
DP-Merging consistently reduces the sharp loss barriers induced by standard DP, yielding smoother paths across vision and language tasks.
}
    \label{fig:bar_all}
    \vspace{-3mm}
\end{figure*}

\subsection{Problem Setup}

\noindent
\textbf{Private Task-specific Fine-tuning.}
Let $w_0 \in \mathbb{R}^d$ denote a public pretrained foundation model.
We consider $T$ downstream tasks, where each task $t \in [T]$ is associated with a private dataset $\mathcal{D}_t$. Starting from the same initialization $w_0$, each task independently applies a randomized fine-tuning algorithm $\mathcal{A}_t$ to obtain a task-specific model $ w_t = \mathcal{A}_t(w_0,\mathcal{D}_t).$ For task $t$, we define the empirical loss as $    \mathcal{L}_t(w)
    :=
    \frac{1}{|\mathcal{D}_t|}
    \sum_{(x,y)\in\mathcal{D}_t}
    \ell(w;x,y),$ which serves as a proxy for the task risk.

\noindent
\textbf{Differential privacy.}
Each task-specific fine-tuning algorithm $\mathcal{A}_t$ is $(\varepsilon,\delta)$-differentially private with respect to its local dataset $\mathcal{D}_t$.
Specifically, for each task $t$, for any neighboring datasets $\mathcal{D}_t$ and $\mathcal{D}'_t$ differing in one example, and any measurable output set $\mathcal{S}$,
\begin{equation}
    \Pr[\mathcal{A}_t(w_0,\mathcal{D}_t) \in \mathcal{S}]
    \le
    e^\varepsilon
    \Pr[\mathcal{A}_t(w_0,\mathcal{D}'_t) \in \mathcal{S}]
    +
    \delta.
\end{equation}
Since each private example contributes to only one task, releasing $\{w_t\}_{t=1}^T$ preserves $(\varepsilon,\delta)$-DP by parallel composition.
Subsequent merging is a post-processing step over the released models and public information, and thus incurs no additional privacy loss.

\noindent
\textbf{Post-hoc Model Merging.}
Given the released task-specific models $\{w_t\}_{t=1}^T$, a post-hoc merging rule constructs a unified model
$
    w_{\mathrm{merge}} = \mathcal{M}(w_1,\dots,w_T).
$
Equivalently, with task vectors $\Delta_t := w_t - w_0$, we write $ w_{\mathrm{merge}}
    =
    w_0
    +
    \mathcal{M}_{\Delta}(\Delta_1,\dots,\Delta_T), $ $\mathcal{M}_{\Delta}$ is the merging rule in task-vector space. 
Our goal is a data-free, privacy-preserving merged model with low average task loss
$
    \frac{1}{T}\sum_{t=1}^T \mathcal{L}_t(w_{\mathrm{merge}}).
$

\vspace{-2mm}
\subsection{Geometric Challenges of Private Mergeability}
\label{subsec:geometry}
We use the merge gap as a diagnostic measure of private mergeability.
For task $t$, the task-wise merge gap is defined as
$\mathrm{Gap}_t(w_{\mathrm{merge}})
    :=
    \mathcal{L}_t(w_{\mathrm{merge}})
    -
    \mathcal{L}_t(w_t).$
A smaller $\mathrm{Gap}_t$ indicates better mergeability, while $\mathcal{L}_t$ is used solely as an evaluation diagnostic and is not available to the post-hoc merging rule.
Consider a second-order expansion of $\mathcal{L}_t$ around the task-specific model $w_t$:
\vspace{-1mm}
\begin{equation}
\label{eq:taylor}
\mathcal{L}_t(w_{\mathrm{merge}})
=
\mathcal{L}_t(w_t)
+
\nabla \mathcal{L}_t(w_t)^\top (w_{\mathrm{merge}}-w_t)
+
\frac{1}{2}
(w_{\mathrm{merge}}-w_t)^\top
H_t
(w_{\mathrm{merge}}-w_t)
+
R_t,
\end{equation}
where $H_t=\nabla^2\mathcal{L}_t(w_t)$ and $R_t$ collects higher-order terms.
This expansion suggests that the merge gap increases with both the local curvature around $w_t$ and the displacement from $w_t$ to $w_{\mathrm{merge}}$:
\begin{equation}
\label{eq:merge_bound}
\mathrm{Gap}_t(w_{\mathrm{merge}})
\approx
\frac{1}{2}
\underbrace{\lambda_{\max}(H_t)}_{\text{local curvature}}
\underbrace{
\|w_{\mathrm{merge}}-w_t\|_2^2
}_{\text{merging displacement}} .
\end{equation}
To visualize the loss barrier induced by merging, we evaluate the loss
along the interpolation path
$w(\gamma)=(1-\gamma)w_t+\gamma w_{\mathrm{merge}}$.
Figure~\ref{fig:bar_all} shows that Naive DP produces sharper and higher loss barriers, whereas DP-Merging yields smoother interpolation paths.
\vspace{-1mm}
\paragraph{Challenge 1: local sharpness.}
DP fine-tuning perturbs the optimization trajectory through gradient clipping and Gaussian noise.
Under tight privacy budgets, the released task model may become more sensitive to parameter perturbations.
This is harmful for merging because the merged model induces a displacement away from the task-specific solution.
We measure this effect using a perturbation-based sharpness proxy:
$
\mathrm{Sharp}_t(w_t;\rho)
=
\mathcal{L}_t
\left(
w_t+
\rho
\frac{\nabla\mathcal{L}_t(w_t)}
{\|\nabla\mathcal{L}_t(w_t)\|_2}
\right)
-
\mathcal{L}_t(w_t)
$.
Figure~\ref{fig:challenge_sharpness} shows that, averaged over the eight vision tasks, Naive DP exhibits larger local sharpness around task-specific models under tighter privacy budgets, while DP-Merging consistently reduces this sensitivity.

\begin{wrapfigure}{r}{0.5\textwidth}
  \centering
  \vspace{-4mm}

  \begin{subfigure}[b]{0.49\linewidth}
    \centering
    \includegraphics[width=\linewidth]{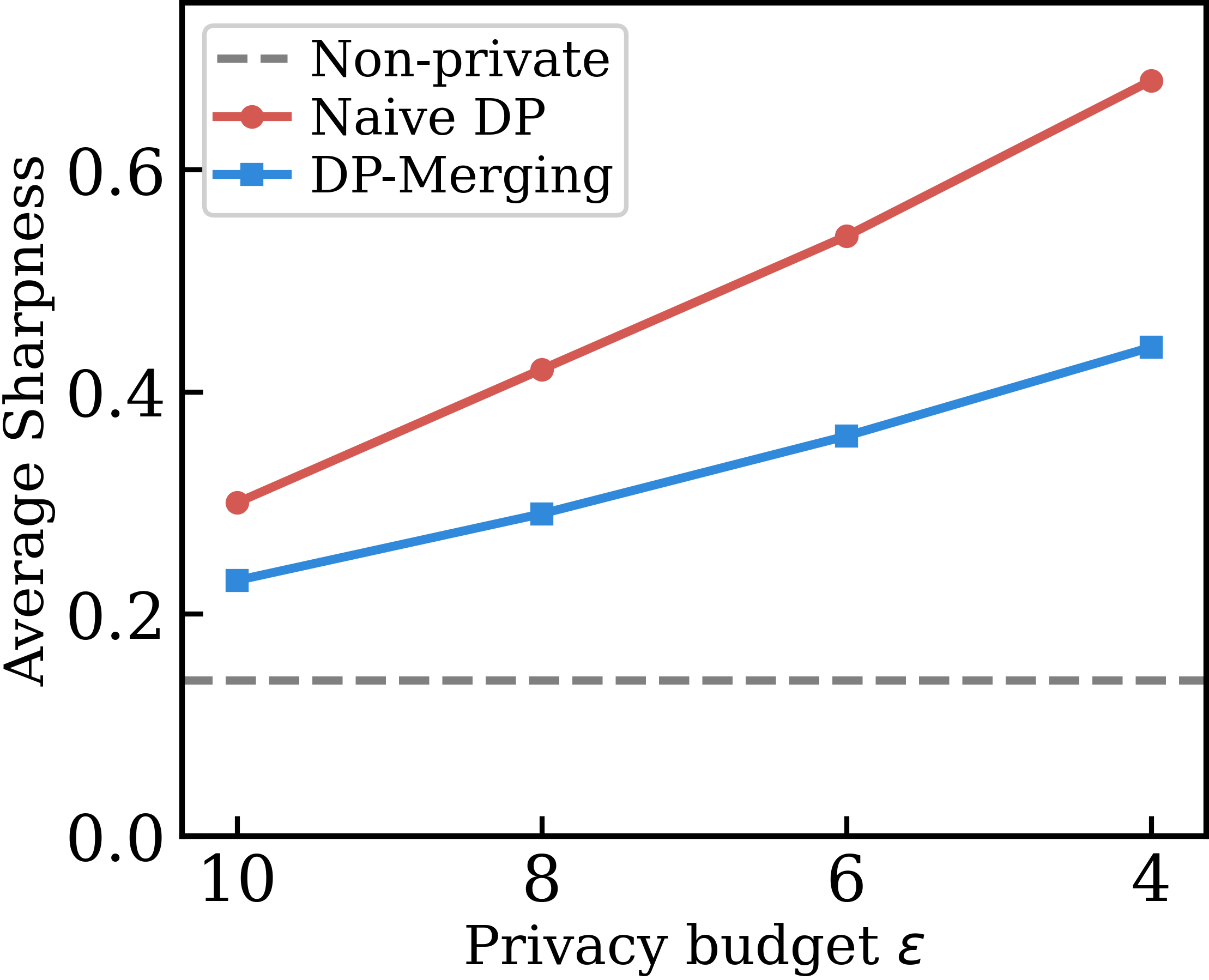}
    \caption{Average Sharpness}
    \label{fig:challenge_sharpness}
  \end{subfigure}
  \hfill
  \begin{subfigure}[b]{0.49\linewidth}
    \centering
    \includegraphics[width=\linewidth]{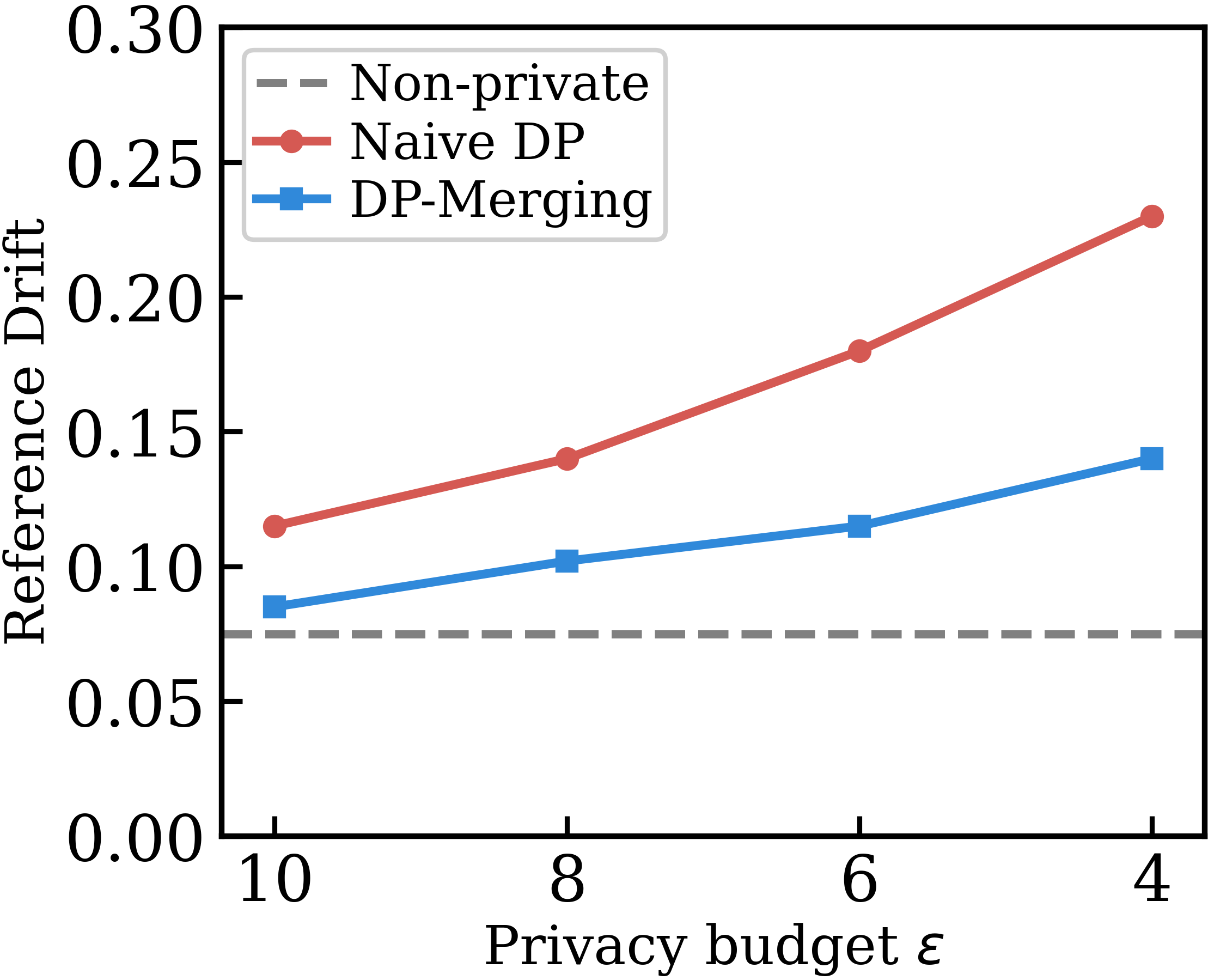}
    \caption{Reference Drift}
    \label{fig:challenge_drift}
  \end{subfigure}
  \caption{
  \small
  Empirical validation on the eight vision tasks using CLIP ViT-B/32.
  }
  \label{fig:challenge_2}
  \vspace{-6mm}
\end{wrapfigure}

\vspace{-1mm}
\paragraph{Challenge 2: reference drift.}
The displacement term in Eq.~\eqref{eq:merge_bound} is affected by how far private task models move away from the shared pretrained initialization. Consider weight averaging, where
$w_{\mathrm{merge}}=w_0+\bar{\Delta}$ and
$\bar{\Delta}=\frac{1}{T}\sum_{j=1}^{T}\Delta_j$.
Then $w_{\mathrm{merge}}-w_t
=
\bar{\Delta}-\Delta_t . $
Thus, merging displacement grows when task updates are large or poorly aligned.
Indeed,
$
\|\bar{\Delta}-\Delta_t\|_2
\le
\frac{1}{T}\sum_{j=1}^{T}\|\Delta_j\|_2+\|\Delta_t\|_2,
$
showing that large drift from $w_0$ can enlarge merge-induced displacement.
We quantify reference drift as
$
\mathrm{Drift}
=
\frac{1}{T}\sum_{t=1}^{T}
\frac{\|w_t-w_0\|_2}{\|w_0\|_2}.
$
As shown in Figure~\ref{fig:challenge_drift}, Naive DP exhibits larger reference drift under stronger privacy constraints, whereas DP-Merging limits this drift.

\section{DP-Merging}
\label{sec:method}
The analysis in Sec.~\ref{subsec:geometry} identifies two obstacles to private mergeability:
(1) local sharpness of private task-specific solutions, and
(2) reference drift from the shared pretrained initialization.
Based on this insight, we propose \textbf{DP-Merging}, a geometry-aware framework for differentially private model merging.
The central idea is to train private task models that are inherently robust to the parameter perturbations induced by merging.
DP-Merging consists of two complementary components:
(1) \emph{DP-compatible sharpness-aware fine-tuning}, which constructs sharpness-aware updates using clipped noisy gradients, and
(2) \emph{reference-anchored alignment}, which limits the displacement between task models through reference anchoring to the pretrained initialization $w_0$.
Intuitively, flatness controls how harmful a displacement is, while anchoring controls how large the displacement becomes.

\subsection{DP-Compatible Sharpness-Aware Fine-Tuning}
\label{subsec:flatness}

We first address the curvature term in Eq.~\eqref{eq:merge_bound}.
When a task model lies in a sharp local region, even a moderate parameter displacement can produce a large increase in task loss after merging.
To mitigate this effect, DP-Merging follows a sharpness-aware principle~\cite{foret2021sharpnessaware}: instead of optimizing the loss only at the current model, it optimizes the loss in a nearby adversarial neighborhood.
As a result, the learned solution becomes locally stable against merge-induced perturbations.

We use flatness not merely to improve the standalone task model, but to make the released model robust to the parameter shift it will undergo during merging. For a model $w$ and minibatch $B$, define the clipped Gaussian gradient estimator
\vspace{-1mm}
\begin{equation}
\label{eq:private_gradient}
\tilde g(w;B)
=
\frac{1}{|B|}
\left(
\sum_{z_i\in B}
\frac{g_i}{\max\{1,\|g_i\|_2/C\}}
+
\mathcal{N}(0,\sigma^2 C^2 I)
\right),
\quad
g_i=\nabla \ell(w;z_i)
\end{equation}
where $C$ is the clipping threshold and $\sigma$ is the noise multiplier. At iteration $k$ for task $t$, we first compute the private gradient $\tilde g_{t,k}=\tilde g(w_{t,k};B_{t,k})$
and use it to construct a local ascent perturbation
\vspace{-1mm}
\begin{equation}
\label{eq:private_sam_eps}
\epsilon_{t,k}
=
\rho_t
\frac{\tilde g_{t,k}}{\|\tilde g_{t,k}\|_2},
\end{equation}
where $\rho_t$ controls the neighborhood size.
We then evaluate a private gradient at the perturbed point:
\vspace{-1mm}
\begin{equation}
\label{eq:private_sam_grad}
\tilde h_{t,k}
=
\tilde g(w_{t,k}+\epsilon_{t,k};B_{t,k}).
\end{equation}

The first gradient $\tilde g_{t,k}$ identifies a nearby high-loss direction, while the second gradient $\tilde h_{t,k}$ updates the model against the loss at that perturbed location.
Consequently, the optimization no longer favors solutions that are only locally optimal at a single point, but instead prefers solutions whose loss remains stable within a neighborhood.
Both gradient evaluations are privatized using clipped Gaussian mechanisms and are accounted for in the privacy analysis.
%The additional sharpness-aware step increases only local computation and does not change the communication structure of model merging.

\subsection{Reference-Anchored Alignment}
\label{subsec:alignment}

Flatness reduces local sensitivity, but it does not ensure that independently trained task models stay close to one another.
We therefore add a simple anchor to the shared initialization.
At iteration $k$, DP-Merging updates task $t$ by
\vspace{-2mm}
\begin{equation}
\label{eq:dp_merging_update}
w_{t,k+1}
=
w_{t,k}
-
\eta
\left(
\tilde h_{t,k}
+
2\lambda(w_{t,k}-w_0)
\right),
\end{equation}
where $\lambda$ controls the alignment strength.
The second term is the gradient of $\lambda\|w_{t,k}-w_0\|_2^2$.
Since it depends only on model parameters and the public initialization, it incurs no additional privacy cost.

This anchor is useful because it controls the displacement term in Eq.~\eqref{eq:merge_bound}.
For uniform averaging, let
$w_{\mathrm{merge}}=\frac{1}{T}\sum_{s=1}^{T}w_s$.
Then
\vspace{-2mm}
\begin{equation}
\label{eq:anchor_displacement}
\scalebox{0.9}{$
\begin{aligned}
\|w_{\mathrm{merge}}-w_t\|_2
&=
\left\|\frac{1}{T}\sum_{s=1}^{T}(w_s-w_t)\right\|_2
\le
\frac{1}{T}\sum_{s=1}^{T}\|w_s-w_t\|_2 \\
&\le
\frac{1}{T}\sum_{s=1}^{T}
\left(\|w_s-w_0\|_2+\|w_t-w_0\|_2\right)
\end{aligned}
$}
\end{equation}
Thus, keeping each task model close to $w_0$ reduces an upper bound on the distance between the merged model and each task-specific solution.
This alignment is induced only through the common reference point and does not require communication between tasks during fine-tuning.

After private fine-tuning, we obtain released task models
$\{w_t^{\mathrm{DP}}\}_{t=1}^{T}$.
For clarity, we use weight averaging as the default merging rule: $w_{\mathrm{merge}}
=
\frac{1}{T}
\sum_{t=1}^{T}
w_t^{\mathrm{DP}}.$
The same released models can also be used with other post-hoc merging rules.
Since the final merging step only processes DP outputs, it is post-processing and preserves the privacy guarantees of task-specific fine-tuning.

\begin{algorithm}[t]
\caption{DP-Merging}
\label{alg:dp_sam_merging}
\renewcommand{\arraystretch}{1.12}
\begin{tabular}{@{}>{\scriptsize}r@{\quad}p{0.88\linewidth}@{}}

1: & \textbf{Input:} pretrained model $w_0$; private datasets $\{\mathcal{D}_t\}_{t=1}^{T}$; iterations $K$; learning rate $\eta$; perturbation radius $\rho_t$; alignment strength $\lambda$; clipping threshold $C$; noise multiplier $\sigma$; sampling rate $q$ \\

2: & \textbf{Output:} merged model $w_{\mathrm{merge}}$ \\[0.3em]

3: & \textbf{for} each task $t=1,\dots,T$ \textbf{do} \\

4: & \quad Initialize task model $w_{t,0}\leftarrow w_0$ \\

5: & \quad \textbf{for} $k=0,\dots,K-1$ \textbf{do} \\

6: & \cellcolor{flatbg}\quad\quad Sample minibatch $B_{t,k}\subset\mathcal{D}_t$ with sampling rate $q$ \\

7: & \cellcolor{flatbg}\quad\quad Compute private gradient $\tilde g_{t,k}\leftarrow \tilde g(w_{t,k};B_{t,k})$ using Eq.~\eqref{eq:private_gradient} \\

8: & \cellcolor{flatbg}\quad\quad Construct perturbation $\epsilon_{t,k}\leftarrow \rho_t\tilde g_{t,k}/\|\tilde g_{t,k}\|_2$ \\

9: & \cellcolor{flatbg}\quad\quad Compute perturbed private gradient $\tilde h_{t,k}\leftarrow \tilde g(w_{t,k}+\epsilon_{t,k};B_{t,k})$ \\

10: & \cellcolor{alignbg}\quad\quad Update with reference anchoring:
$w_{t,k+1}
\leftarrow
w_{t,k}
-
\eta
\bigl(
\tilde h_{t,k}
+
2\lambda(w_{t,k}-w_0)
\bigr)$ \\

11: & \quad \textbf{end for} \\

12: & \quad Obtain private task model $w_t^{\mathrm{DP}}\leftarrow w_{t,K}$ \\

13: & \textbf{end for} \\[0.2em]

14: & $w_{\mathrm{merge}}\leftarrow \frac{1}{T}\sum_{t=1}^{T}w_t^{\mathrm{DP}}$ \\

15: & \textbf{return} $w_{\mathrm{merge}}$ \\

\end{tabular}
\end{algorithm}
\section{Theoretical Analysis}\label{sec:theory}
\subsection{Privacy guarantee of DP-Merging}
\label{subsec:privacy_dp_merging}

We analyze the privacy guarantee of DP-Merging in
Appendix~\ref{app:privacy}. Each iteration uses two privatized
clipped-gradient evaluations, while the sharpness perturbation, reference anchor,
and final merging step are post-processing operations. Therefore, they incur no
additional privacy loss beyond the underlying private gradient evaluations.

\begin{theorem}[Privacy guarantee of DP-Merging]
\label{thm:privacy_budget}
Assume that each task dataset $\mathcal{D}_t$ is disjoint, and each private
example belongs to at most one task. For task $t$, suppose Algorithm~\ref{alg:dp_sam_merging}
runs for $K$ iterations with Poisson sampling rate $q$, clipping threshold $C$,
and Gaussian noise multiplier $\sigma$. Let
$\varepsilon_{\mathrm{pair}}(\alpha;q,\sigma)$ denote the order-$\alpha$ RDP~\cite{mironov2017renyi} cost
of one Poisson-subsampled paired Gaussian mechanism that releases the two noisy
clipped-gradient quantities used in one DP-Merging iteration. Then, for any
R\'enyi order $\alpha>1$, the released task model $w_t^{\mathrm{DP}}$ satisfies
$
(\alpha,\, K\varepsilon_{\mathrm{pair}}(\alpha;q,\sigma))\text{-RDP}.
$
Consequently, for any $\delta_t\in(0,1)$, $w_t^{\mathrm{DP}}$ satisfies
$(\varepsilon_t,\delta_t)$-DP with
\vspace{-1mm}
\begin{equation}
\small
\varepsilon_t
=
\min_{\alpha>1}
\left\{
K\varepsilon_{\mathrm{pair}}(\alpha;q,\sigma)
+
\frac{\log(1/\delta_t)}{\alpha-1}
\right\}.
\end{equation}
Since the task datasets are disjoint, releasing all private task models
$\{w_t^{\mathrm{DP}}\}_{t=1}^{T}$ satisfies
$
\left(
\max_{t\in[T]}\varepsilon_t,\,
\max_{t\in[T]}\delta_t
\right)\text{-DP}.
$
The merged model
$
w_{\mathrm{merge}}
=
\frac{1}{T}\sum_{t=1}^{T} w_t^{\mathrm{DP}}
$
incurs no additional privacy loss by post-processing.
\end{theorem}

\subsection{Mergeability Analysis}
\label{sec:theory}

We provide a theoretical analysis explaining why DP-Merging improves the mergeability of
differentially private task models. %The analysis formalizes the intuition in Sec.~\ref{sec:preliminaries}: the loss increase caused by merging is controlled by two factors, namely the local curvature around each private task model and the distance between the merged model and that task model. DP-Merging acts on these two factors directly: the sharpness-aware update encourages smaller local curvature, while the reference anchor reduces task-vector drift from the shared initialization.
Let $w_t := w_t^{\mathrm{DP}}$ be the private task model returned by Algorithm~\ref{alg:dp_sam_merging},
and define the task vector $u_t := w_t-w_0$. For weight averaging,
$
\bar u := \frac{1}{T}\sum_{s=1}^{T}u_s,
w_{\mathrm{merge}} := w_0+\bar u,
\Delta_t := w_{\mathrm{merge}}-w_t=\bar u-u_t .
$
We measure mergeability by the average merge gap
\vspace{-2mm}
\begin{equation}
\small
G_{\mathrm{merge}}
:=
\frac{1}{T}\sum_{t=1}^{T}
\bigl(\mathcal{L}_t(w_{\mathrm{merge}})-\mathcal{L}_t(w_t)\bigr).
\end{equation}

\begin{theorem}[Mergeability bound for DP-Merging]
\label{thm:mergeability}
Assume that, for each task $t$, $\mathcal{L}_t$ is three-times differentiable along the segment
$\{w_t+\gamma\Delta_t:\gamma\in[0,1]\}$. Let
$
\beta_t :=
\sup_{\gamma\in[0,1]}
\lambda_{\max}\!\left(\nabla^2 \mathcal{L}_t(w_t+\gamma\Delta_t)\right),
$
and assume that the third-order Taylor remainder is bounded by
$
|R_t|
\le
\frac{M_t}{6}\|\Delta_t\|_2^3 .
$
If the returned private task model satisfies
$
\|\nabla \mathcal{L}_t(w_t)\|_2 \le \varepsilon_t ,
$
then
\vspace{-2mm}
\begin{equation}
\small
G_{\mathrm{merge}}
\le
\frac{1}{T}\sum_{t=1}^{T}
\left[
\varepsilon_t\|\Delta_t\|_2
+
\frac{\beta_t}{2}\|\Delta_t\|_2^2
+
\frac{M_t}{6}\|\Delta_t\|_2^3
\right].
\end{equation}
\end{theorem}

Theorem~\ref{thm:mergeability} shows that the merge gap decreases when either the local curvature
$\beta_t$ is small or the merging displacement $\|\Delta_t\|_2$ is small. This directly matches the
two design choices of DP-Merging. The DP-compatible sharpness-aware update reduces sensitivity
to local perturbations, thereby targeting the curvature term. The reference-anchored regularizer
controls the task-vector norm $\|w_t-w_0\|_2$, which in turn controls $\|\Delta_t\|_2$.

To make the role of the anchor explicit, define the robust loss
$
L_{t,\rho_t}(w)
:=
\max_{\|\epsilon\|_2\le \rho_t}
\mathcal{L}_t(w+\epsilon),
$
and the regularized robust objective approximately optimized by DP-Merging:
$
\Phi_t(w)
:=
L_{t,\rho_t}(w)
+
\lambda\|w-w_0\|_2^2 .
$
Suppose the returned model satisfies
$
\|\nabla L_{t,\rho_t}(w_t)\|_2 \le G_t,
\|\nabla \Phi_t(w_t)\|_2 \le \zeta_t .
$
Then the anchor gives
$
\|w_t-w_0\|_2
\le
\frac{G_t+\zeta_t}{2\lambda}.
$
Consequently, if
$
A_t := \frac{G_t+\zeta_t}{2\lambda},
\bar A := \frac{1}{T}\sum_{s=1}^{T}A_s,
$
then
$
\|\Delta_t\|_2
=
\|\bar u-u_t\|_2
\le
A_t+\bar A .
$
Substituting this into Theorem~\ref{thm:mergeability} yields
\vspace{-2mm}
\begin{equation}
\small
G_{\mathrm{merge}}
\le
\frac{1}{T}\sum_{t=1}^{T}
\left[
\varepsilon_t(A_t+\bar A)
+
\frac{\beta_t}{2}(A_t+\bar A)^2
+
\frac{M_t}{6}(A_t+\bar A)^3
\right].
\end{equation}
In the uniform case where
$\varepsilon_t\le\varepsilon$, $\beta_t\le\beta$, $M_t\le M$,
$G_t\le G$, and $\zeta_t\le\zeta$ for all $t$, we obtain the simplified bound
\vspace{-2mm}
\begin{equation}
\small
G_{\mathrm{merge}}
\le
\varepsilon\frac{G+\zeta}{\lambda}
+
\frac{\beta}{2}\left(\frac{G+\zeta}{\lambda}\right)^2
+
\frac{M}{6}\left(\frac{G+\zeta}{\lambda}\right)^3 .
\end{equation}
This bound explains DP-Merging’s improved private mergeability: the sharpness-aware component
reduces $\beta$, and the anchor increases geometric compatibility by reducing the displacement
scale $(G+\zeta)/\lambda$.

\begin{table*}[h]
\centering
\caption{Multi-task accuracy (\%) on the 8-task vision benchmark with ViT-B/32, $\varepsilon=4$.}
\label{tab:performance_vitbase32}
\renewcommand\arraystretch{0.95}
\setlength{\tabcolsep}{5pt}
\resizebox{\linewidth}{!}{
\begin{tabular}{l|cccccccc|c}
\toprule
\textbf{Method}
& \textbf{SUN397}
& \textbf{Cars}
& \textbf{RESISC45}
& \textbf{EuroSAT}
& \textbf{SVHN}
& \textbf{GTSRB}
& \textbf{MNIST}
& \textbf{DTD}
& \textbf{Avg Acc} \\
\midrule
DP + Individual
& 59.2 & 67.7 & 86.1 & 89.7 & 87.5 & 88.6 & 89.3 & 69.4 & 79.7 \\
\midrule
DP + Weight Averaging
& 45.3 & 43.4 & 51.9 & 52.7 & 43.2 & 33.8 & 67.1 & 30.9 & 46.0 \\
\rowcolor{flatbg}
DP-Merging + Weight Averaging
& 48.2 & 49.5 & 56.3 & 57.5 & 43.2 & 37.4 & 69.0 & 36.3 & $49.7_{\textcolor{red}{+3.7}}$ \\

DP + RegMean
& 45.8 & 44.9 & 58.3 & 58.6 & 59.1 & 48.3 & 73.7 & 32.0 & 52.6 \\
\rowcolor{flatbg}
DP-Merging + RegMean
& 49.4 & 48.6 & 63.1 & 62.4 & 63.6 & 53.1 & 76.8 & 38.2 & $56.9_{\textcolor{red}{+4.3}}$ \\

DP + Task Arithmetic
& 45.2 & 44.9 & 56.7 & 58.9 & 60.2 & 59.7 & 77.3 & 30.4 & 54.2 \\
\rowcolor{flatbg}
DP-Merging + Task Arithmetic
& 48.9 & 49.2 & 59.6 & 62.4 & 67.1 & 62.5 & 79.6 & 35.1 & $58.1_{\textcolor{red}{+3.9}}$ \\

DP + TIES-Merging
& 49.8 & 48.6 & 50.7 & 59.7 & 66.5 & 72.1 & 78.3 & 34.2 & 57.5 \\
\rowcolor{flatbg}
DP-Merging + TIES-Merging
& 52.9 & 53.7 & 56.0 & 64.1 & 72.1 & 78.0 & 81.4 & 38.6 & $62.1_{\textcolor{red}{+4.6}}$\\

DP + PCB Merging
& 45.1 & 43.4 & 50.5 & 57.2 & 61.7 & 70.9 & 77.0 & 30.3 & 54.5 \\
\rowcolor{flatbg}
DP-Merging + PCB Merging
& 49.5 & 48.2 & 56.4 & 63.4 & 65.0 & 76.3 & 79.1 & 35.4 & $59.2_{\textcolor{red}{+4.7}}$ \\

DP + WUDI-Merging
& 47.6 & 45.5 & 58.5 & 59.3 & 66.4 & 57.1 & 78.2 & 39.1 & 56.5 \\
\rowcolor{flatbg}
DP-Merging + WUDI-Merging
& 51.5 & 48.6 & 63.1 & 63.9 & 70.4 & 61.2 & 82.7 & 44.8 & $60.8_{\textcolor{red}{+4.3}}$ \\
\bottomrule
\end{tabular}
}
\end{table*}
\begin{table*}[h]
\centering
\caption{Multi-task accuracy (\%) on the 8-task vision benchmark with ViT-L/14, $\varepsilon=4$.}
\label{tab:performance_vit_l_14} 
\renewcommand\arraystretch{0.95}
\setlength{\tabcolsep}{5pt}
\resizebox{\linewidth}{!}{  
\begin{tabular}{l|cccccccc|c}
\toprule
\textbf{Method}  &  \textbf{SUN397}  &  \textbf{Cars}  &  \textbf{RESISC45}  &  \textbf{EuroSAT}  &  \textbf{SVHN}  &  \textbf{GTSRB}  &  \textbf{MNIST}  &  \textbf{DTD}  & \textbf{Avg Acc}  \\
\midrule
DP + Individual
& 64.8 & 74.2 & 93.1 & 92.9 & 92.2 & 95.4 & 96.7 & 75.6 & 85.5 \\
\midrule
DP + Weight Averaging
& 50.3 & 49.4 & 57.4 & 57.7 & 50.2 & 39.8 & 72.5 & 35.1 & 51.6 \\
\rowcolor{flatbg}
DP-Merging + Weight Averaging
& 54.4 & 54.4 & 62.1 & 62.2 & 55.4 & 45.4 & 80.1 & 44.2 & $57.3_{\textcolor{red}{+5.7}}$ \\

DP + RegMean
& 50.3 & 49.8 & 61.0 & 64.6 & 63.2 & 53.4 & 79.7 & 41.6 & 57.9 \\
\rowcolor{flatbg}
DP-Merging + RegMean
& 53.3 & 54.8 & 65.9 & 70.6 & 69.2 & 55.0 & 82.7 & 49.6 & $62.7_{\textcolor{red}{+4.8}}$ \\

DP + Task Arithmetic
& 50.2 & 50.9 & 61.7 & 73.5 & 75.4 & 54.7 & 82.3 & 45.4 & 61.8 \\
\rowcolor{flatbg}
DP-Merging + Task Arithmetic
& 54.1 & 55.5 & 64.4 & 78.9 & 84.1 & 62.5 & 84.1 & 51.3 & $66.7_{\textcolor{red}{+4.9}}$ \\

DP + TIES-Merging
& 44.8 & 45.3 & 46.8 & 64.4 & 71.2 & 58.7 & 83.9 & 50.2 & 58.1 \\
\rowcolor{flatbg}
DP-Merging + TIES-Merging
& 50.8 & 52.6 & 51.7 & 69.7 & 75.3 & 61.1 & 87.8 & 52.8 & $62.7_{\textcolor{red}{+4.6}}$ \\

DP + PCB Merging
& 50.7 & 49.6 & 51.5 & 62.2 & 67.7 & 55.9 & 82.6 & 42.1 & 57.8 \\
\rowcolor{flatbg}
DP-Merging + PCB Merging
& 54.5 & 54.2 & 56.2 & 67.9 & 71.2 & 61.3 & 87.2 & 47.4 & $62.5_{\textcolor{red}{+4.7}}$ \\

DP + WUDI-Merging
& 51.7 & 51.1 & 63.4 & 64.3 & 71.8 & 63.1 & 83.0 & 44.2 & 61.6 \\
\rowcolor{flatbg}
DP-Merging + WUDI-Merging
& 54.9 & 54.7 & 66.8 & 68.5 & 76.6 & 68.3 & 85.1 & 50.3 & $65.6_{\textcolor{red}{+4.0}}$ \\
\bottomrule
\end{tabular}
}
\vspace{-2mm}
\end{table*}

\section{Experiments}
\label{sec:experiments}

\subsection{Experimental Settings}
\label{sec:Experimental Settings}
\noindent
\textbf{Datasets and Models.} We evaluate our method on \textit{vision and language tasks.} 
For vision tasks, following~\cite{yang2024adamerging,cheng2025whoever}, we study multi-task model merging across eight image classification datasets, namely SUN397~\citep{xiao2016sun}, Cars~\citep{krause20133d}, RESISC45~\citep{cheng2017remote}, EuroSAT~\citep{helber2019eurosat}, SVHN~\citep{yuval2011reading}, GTSRB~\citep{stallkamp2011german}, MNIST~\citep{lecun1998mnist}, and DTD~\citep{cimpoi2014describing}, and adopt CLIP-based ViT backbones~\citep{radford2021learning}, including ViT-B/32, ViT-B/16, and ViT-L/14. For language tasks, we use the GLUE benchmark~\citep{wang2018glue} and evaluate with RoBERTa-Base and RoBERTa-Large~\citep{liu2019roberta}.

\noindent
\textbf{Baselines.}
We compare with representative merging methods under the same DP setting, including Weight Averaging~\cite{wortsman2022model}, RegMean~\cite{jin2023dataless}, Task Arithmetic~\cite{ilharco2023editing}, TIES-Merging~\cite{yadav2023tiesmerging}, PCB Merging~\cite{du2024parameter}, and WUDI-Merging~\cite{cheng2025whoever}.
All baselines independently train task-specific models with the same DP optimizer and privacy budget, and then merge the released DP models without accessing private data.
The merging step is post-processing and therefore incurs no additional privacy loss.
We apply the same operators to the task models produced by DP-Merging to evaluate whether our merge-aware fine-tuning improves mergeability across merging rules.

\noindent
\textbf{Implementation Details.}
We implement sample-level DP fine-tuning with Opacus~\cite{yousefpour2021opacus} and use AdamW~\cite{loshchilov2017fixing} as the default optimizer.
Unless otherwise specified, each task model is trained with clipping norm $C=0.2$, batch size $16$, learning rate $1e-5$, weight decay $1e-2$, and $10$ training epochs/steps.
Privacy loss is computed using the Opacus privacy accountant and reported as $(\varepsilon,\delta)$-DP, where $\delta=1/N$ and $\varepsilon\in\{1,2,4,8\}$.
For DP-Merging, the perturbation radius $\rho=0.05$ and the alignment coefficient $\lambda=0.05$.
Additional implementation details are provided in Appendix~\ref{app:Implementation Details}.

\subsection{Main Results}
\label{Main Results}
\noindent
\textbf{Evaluation on Visual Tasks.}
We evaluate DP-Merging on the 8-task vision benchmark under $\varepsilon=4$.
Table~\ref{tab:performance_vitbase32} and Table~\ref{tab:performance_vit_l_14} report the main results with ViT-B/32 and ViT-L/14, and additional results with ViT-B/16 are provided in Appendix~\ref{app:More Results on Vision Tasks}.
Across both backbones, standard merging methods exhibit a clear drop from the DP individual reference, suggesting that DP fine-tuning hurts task-vector mergeability.
DP-Merging consistently improves the merged accuracy under all merging operators.
For instance, on ViT-B/32, WUDI-Merging improves from $56.5\%$ to $60.8\%$ when applied to DP-Merging task models. The same trend holds for ViT-L/14 and ViT-B/16.

\noindent
\textbf{Evaluation on Language Tasks.}
Table~\ref{tab:performance_roberta_large} summarizes the multi-task performance of RoBERTa-Large models under $\varepsilon=4$. Compared with standard DP baselines, DP-Merging consistently achieves higher average accuracy across different merging operators. For example, the average score for DP-Merging ranges from 73.9 to 77.0, depending on the merging method used. The average score for naive DP ranges from 71.6 to 74.7. These results indicate that merge-aware DP fine-tuning consistently maintains strong performance across multiple language understanding tasks.

\begin{table*}[h]
\centering
\caption{Multi-task performance of RoBERTa-Large models on GLUE benchmark, $\varepsilon=4$.}
\label{tab:performance_roberta_large} 
\renewcommand\arraystretch{0.95}
\setlength{\tabcolsep}{10pt}
\resizebox{\linewidth}{!}{  
\begin{tabular}{l|cccccccc|cc}
\toprule
\textbf{Method} & \textbf{CoLA} & \textbf{MNLI} & \textbf{MRPC} & \textbf{QNLI} & \textbf{QQP}  & \textbf{RTE}  & \textbf{SST2} & \textbf{STSB} & \textbf{Avg.} \\
\midrule
%\multicolumn{10}{c}{\emph{Non-merging Methods}} \\
DP + Individual
& 60.5 & 79.1 & 81.0 & 89.2 & 82.3 & 77.1 & 93.8 &82.1& 80.6 \\
\midrule
DP + Weight Averaging
& 56.1 & 73.8 & 67.9 & 79.5 & 78.6 & 69.3 & 91.2 & 67.6 & 72.6 \\
\rowcolor{flatbg}
DP-Merging + Weight Averaging
& 58.2 & 76.6 & 69.2 & 79.3 & 81.6 & 69.9 & 91.8 & 68.2 & $73.9_{\textcolor{red}{+1.3}}$ \\

DP + RegMean
& 55.6 & 73.8 & 65.7 & 79.2 & 78.6 & 67.9 & 91.0& 62.5 & 71.6 \\
\rowcolor{flatbg}
DP-Merging + RegMean
& 58.8 & 78.7 & 68.6& 79.5 & 80.9 & 69.3 & 91.9 & 64.6 & $74.7_{\textcolor{red}{+3.1}}$ \\

DP + Task Arithmetic
& 56.1 & 75.5 & 67.2 & 78.9 & 81.5 & 69.9 & 91.2 & 61.7 & 71.8 \\
\rowcolor{flatbg}
DP-Merging + Task Arithmetic
& 59.6 & 77.8 & 68.9 & 81.5 & 82.4 & 70.9 & 91.9 & 63.5 & $74.2_{\textcolor{red}{+2.4}}$ \\

DP + TIES-Merging
& 60.5 & 79.3 & 70.8 & 80.6 & 81.3 & 70.6 & 91.1 & 64.3 & 74.3 \\
\rowcolor{flatbg}
DP-Merging + TIES-Merging
& 61.0 & 81.2 & 73.3 & 82.5 & 82.6 & 71.2 & 92.6 & 65.9 & $76.5_{\textcolor{red}{+2.2}}$ \\

DP + PCB Merging
& 56.2 & 74.2 & 76.6 & 77.5 & 81.8 & 70.6 & 91.1 & 60.2 & 73.8 \\
\rowcolor{flatbg}
DP-Merging + PCB Merging
& 58.9 & 75.7 & 79.4 & 79.8 & 83.1 & 72.4 & 92.6 & 61.8 & $75.9_{\textcolor{red}{+2.1}}$ \\

DP + WUDI-Merging
& 57.2 & 75.8 & 78.9 & 79.6 & 82.6 & 73.5 & 91.3 & 59.5 & 74.7 \\
\rowcolor{flatbg}
DP-Merging + WUDI-Merging
& 59.8 & 77.9 & 82.1 & 83.3 & 82.7 & 75.1 & 92.8 & 61.3 & $77.0_{\textcolor{red}{+2.3}}$ \\
\bottomrule
\end{tabular}
}
\vspace{-2mm}
\end{table*}

\begin{table*}[h]
\small
\centering
\caption{Performance under different privacy budgets on visual and language tasks using TIES-Merging. 
Smaller $\varepsilon$ indicates stronger privacy protection.}
\label{tab:privacy_budget}
\setlength{\tabcolsep}{9pt}
\renewcommand{\arraystretch}{0.9}
\small
\resizebox{0.95\textwidth}{!}{
\begin{tabular}{c|ccc|ccc}
\toprule
\multirow{2}{*}{\textbf{Privacy Budget $\varepsilon$}} 
& \multicolumn{3}{c|}{\textbf{Visual Tasks (ViT-B/32)}} 
& \multicolumn{3}{c}{\textbf{Language Tasks (RoBERTa-Large)}} \\
\cmidrule(lr){2-4} \cmidrule(lr){5-7}
& \textbf{DP(AdamW)} & \textbf{DP-Merging} & \textbf{Gain} 
& \textbf{DP(AdamW)} & \textbf{DP-Merging} & \textbf{Gain} \\
\midrule
1.0  & 46.1 & \textbf{52.2} & +6.1 & 64.6 & \textbf{68.4} & +3.8 \\
2.0  & 52.8 & \textbf{58.6} & +5.8 & 68.3 & \textbf{70.6} & +2.3 \\
4.0  & 57.5 & \textbf{62.1} & +4.6 & 74.3 & \textbf{76.5} & +2.2 \\
8.0  & 62.0 & \textbf{66.7} & +4.7 & 76.3 & \textbf{77.8} & +1.5 \\
\bottomrule
\end{tabular}
}
\vspace{-2mm}
\end{table*}
\vspace{-2mm}
\subsection{Analysis under Different Privacy Budgets}
\label{subsec:privacy_budget}
\vspace{-2mm}
To evaluate DP-Merging under varying privacy budgets $\varepsilon$, we keep the merging protocol unchanged.
A smaller $\varepsilon$ corresponds to stronger privacy protection and usually introduces larger optimization perturbations through clipping and noise. 
As shown in Table~\ref{tab:privacy_budget}, DP-Merging consistently outperforms naive DP merging across different privacy budgets. 
The improvement is more pronounced under stronger privacy constraints, suggesting that the proposed flatness and alignment components effectively mitigate the geometry distortion caused by differential privacy.

\vspace{-2mm}
\subsection{Ablation Study}
\vspace{-2mm}
We conduct ablation studies to understand the contribution of each component in DP-Merging and analyze its sensitivity to key hyperparameters.
All experiments use $\varepsilon=4$ and default settings.

\textbf{Effect of each component.}
DP-Merging contains two key components: sharpness-aware fine-tuning and reference-anchored alignment. To evaluate their individual contributions, we compare the full method with two variants:
(1) removing the alignment component, and
(2) removing the sharpness-aware fine-tuning.
As shown in Table~\ref{tab:ablation_components}, both components improve performance over naive DP. The sharpness-aware fine-tuning mainly improves local robustness by encouraging flatter task-specific solutions, while the alignment component enhances cross-task compatibility by reducing discrepancies among task vectors. Combining both components yields the best performance, demonstrating that sharpness and alignment are complementary.

\begin{table}[h]
\vspace{-2mm}
\small
\centering
\caption{Ablation study of different merging algorithms on the 8-task vision benchmark with ViT-B/32. ``w/o Alignment'' represents without the reference-anchored alignment component, and ``w/o Flatness'' represents without sharpness-aware fine-tuning.}
\label{tab:ablation_components}
\setlength{\tabcolsep}{3pt}
\renewcommand{\arraystretch}{0.95}
\small
\resizebox{0.9\linewidth}{!}{
\begin{tabular}{lcccccc}
\toprule
\textbf{Variant} 
& \textbf{Weight Averaging} 
& \textbf{Task Arithmetic} 
& \textbf{TIES-Merging}  
& \textbf{PCB-Merging} \\
\midrule
DP-AdamW 
& 46.0 
& 54.2
& 57.5
& 54.5 \\
DP-Merging 
& 49.7 
& 58.1
& 62.1
& 59.2 \\
DP-Merging w/o Alignment 
& 48.1 
& 55.6 
& 60.7  
& 57.4 \\
DP-Merging w/o Flatness 
& 47.5 
& 56.2 
& 59.8  
& 57.9 \\
\bottomrule
\end{tabular}
}
\end{table}

\textbf{Sensitivity to the flatness radius $\rho$.}
We study the influence of the flatness radius $\rho$ on merged accuracy. When varying $\rho$, the alignment strength $\lambda$ is fixed to its default value. As shown in Table~\ref{tab:ablation_rho}, moderate values of $\rho$ consistently improve merged accuracy, while overly large values may hurt task-specific adaptation or over-constrain different tasks.

\textbf{Sensitivity to the alignment strength $\lambda$.}
We study the effect of alignment strength $\lambda$ on merged accuracy (flatness radius $\rho$ fixed). As shown in Table~\ref{tab:ablation_lambda}, moderate values of $\lambda$ consistently improve merged accuracy, while excessively large values may limit flexibility across tasks.

\begin{table*}[h]
\small
\centering
\begin{minipage}{0.48\linewidth}
\small
\centering
\captionof{table}{Sensitivity to the flatness radius $\rho$ on ViT-B/32 using TIES-Merging.}
\label{tab:ablation_rho}
\renewcommand\arraystretch{0.75}
\setlength{\tabcolsep}{6pt}
\begin{tabular}{lccccc}
\toprule
$\rho$ & 0 & 0.01 & 0.03 & 0.05 & 0.10 \\
\midrule
Avg Acc & 59.8 & 60.4 & 61.6 & \textbf{62.1} & 61.8 \\
\bottomrule
\end{tabular}
\end{minipage}
\hfill
\begin{minipage}{0.48\linewidth}
\small
\centering
\captionof{table}{Sensitivity to the alignment strength $\lambda$ on ViT-B/32 using TIES-Merging.}
\label{tab:ablation_lambda}
\renewcommand\arraystretch{0.75}
\setlength{\tabcolsep}{6pt}
\begin{tabular}{lccccc}
\toprule
$\lambda$ & 0 & 0.01 & 0.05 & 0.10 & 0.50 \\
\midrule
Avg Acc & 60.8 & 61.2 & \textbf{62.3} & 62.0 & 61.6 \\
\bottomrule
\end{tabular}
\end{minipage}

\end{table*}

\textbf{Robustness to merging operators.}
Since our method improves the private fine-tuning stage rather than designing a new merging rule, we evaluate whether it works with different data-free merging operators.
As shown in Table~\ref{tab:ablation_merger}, DP-Merging consistently improves the final merged model across different representative merging operators, demonstrating that our method enhances the intrinsic mergeability of task-specific models.

\begin{table}[h]
\vspace{-2mm}
\small
\centering
\caption{Robustness to different merging operators on the 8-task vision benchmark with ViT-B/32.}
\label{tab:ablation_merger}
\renewcommand\arraystretch{0.9}
\setlength{\tabcolsep}{1.5pt}
\begin{tabular}{lccccc}
\toprule
\textbf{Training} & \textbf{Weight Averaging} & \textbf{Task Arithmetic} & \textbf{TIES-Merging} &\textbf{PCB-Merging} & \textbf{WUDI-Merging} \\
\midrule
DP(AdamW) 
& 46.0 & 54.2 & 57.5 & 54.5 & 56.5 \\
\rowcolor{flatbg}
\textbf{DP-Merging (Ours)}
 & 49.7 & 58.1 & 62.1 & 59.2 & 60.8\\
\bottomrule
\end{tabular}
\vspace{-2mm}
\end{table}

\section{Conclusion}
\label{sec:conclusion}
We studied differentially private model merging, where independently fine-tuned private models are combined without sharing task data. The challenge arises not only from the utility loss of DP fine-tuned models, but also from their geometric incompatibility, caused by local curvature and displacement from the pretrained anchor. To address this, we propose DP-Merging, which combines a flatness objective with a pretrained-anchor regularizer to improve robustness and control task-specific drift. Experiments on vision and language benchmarks show that DP-Merging consistently improves merged-model performance under different privacy budgets and merging settings. We hope this work contributes to a better understanding of, and advances in, model merging under privacy constraints.

\newpage
\bibliographystyle{plain}
\bibliography{main}

\onecolumn
\setcounter{page}{1}
\thispagestyle{empty}

\begin{center}
    {\LARGE\bfseries When Privacy Hurts Mergeability: Geometry-Aware Model Merging under Differential Privacy\par}
    \vspace{3mm}
    {\Large Supplementary Material\par}
\end{center}

\vspace{8mm}

\begin{center}
{\Large\bfseries LIST OF APPENDICES\par}
\end{center}

\vspace{3mm}

\begin{center}
\begin{minipage}{0.72\textwidth}
{\large
\noindent \hyperref[app:Notation]{\textbf{A: Notation}}

\noindent \hyperref[app:exp_all]{\textbf{B: Implementation of Experiments}}
\begin{enumerate}[label=B.\arabic*, leftmargin=3.2em, itemsep=1pt, topsep=2pt]
  \item \hyperref[app:Datasets]{Datasets}
  \item \hyperref[app:Models]{Models}
  \item \hyperref[app:Implementation Details]{Implementation Details}
  \item \hyperref[app:More Results on Vision Tasks]{More Results on Vision Tasks}
  \item \hyperref[app:More Results on Language Tasks]{More Results on Language Tasks}
\end{enumerate}

\vspace{1mm}
\noindent \hyperref[app:theoretical]{\textbf{C: Implementation of Theoretical Analysis}}
\begin{enumerate}[label=C.\arabic*, leftmargin=3.2em, itemsep=1pt, topsep=2pt]
  \item \hyperref[app:privacy]{Privacy guarantee of DP-Merging}
  \item \hyperref[app:mergeability]{Mergeability Analysis}
  \item \hyperref[app:anchor]{Reference Anchoring Controls Merging Displacement}
  \item \hyperref[app:jtl]{Connection to Joint-Task Loss Linearity}
\end{enumerate}

\vspace{1mm}
\noindent \hyperref[app:Limitations and Future Work]{\textbf{D: Limitations and Future Work}}
}
\end{minipage}
\end{center}

\newpage
\appendix

\section{Notation}
\label{app:Notation}
Table~\ref{tab:notation} summarizes the main notation used throughout the paper.

\begin{table}[h]
\centering
\caption{Summary of notation.}
\label{tab:notation}
\small
\setlength{\tabcolsep}{5pt}
\renewcommand{\arraystretch}{1.08}
\begin{tabular}{l|p{0.72\linewidth}}
\toprule
\textbf{Notation} & \textbf{Description} \\
\midrule
$w_0$ & Public pretrained initialization shared by all downstream tasks. \\
$T$ & Number of downstream tasks. \\
$t \in [T]$ & Task index. \\
$\mathcal{D}_t$ & Private training dataset of task $t$. \\
$\ell(w;x,y)$ & Sample-wise loss evaluated at model parameter $w$ on example $(x,y)$. \\
$\mathcal{L}_t(w)$ & Empirical loss of task $t$ on $\mathcal{D}_t$. \\
$\mathcal{A}_t$ & Randomized DP fine-tuning algorithm for task $t$. \\
$w_t$ & Task-specific model obtained by privately fine-tuning $w_0$ on $\mathcal{D}_t$. \\
$\Delta_t = w_t - w_0$ & Task vector of task $t$ relative to the public initialization. \\
$w_{\mathrm{merge}}$ & Merged model obtained from the released task-specific models. \\
$\mathcal{M}$ & Model-merging operator, such as Weight Averaging, RegMean, Task Arithmetic, TIES-Merging, PCB, or WUDI. \\
$H_t$ & Hessian or local curvature matrix of $\mathcal{L}_t$ around $w_t$. \\
$\lambda_{\max}(H_t)$ & Largest eigenvalue of $H_t$, used as a local sharpness measure. \\
$\|w_t-w_0\|_2$ & Reference drift of task model $w_t$ from the shared initialization $w_0$. \\
$C$ & Per-sample gradient clipping threshold in DP fine-tuning. \\
$\sigma$ & Noise multiplier of the Gaussian mechanism. \\
$q$ & Sampling rate used in DP fine-tuning. \\
$(\varepsilon,\delta)$ & Differential privacy parameters. \\
$\rho$ & SAM radius used in the DP-aware flat-minima search component. \\
$\lambda$ & Coefficient of the reference-based alignment regularizer. \\
\bottomrule
\end{tabular}
\end{table}

\section{Implementation of Experiments}
\label{app:exp_all}

\subsection{Datasets}
\label{app:Datasets}

We evaluate DP-Merging on both vision and language benchmarks.
For all tasks, the training data are treated as private and are only used during DP fine-tuning.
The merging stage only accesses the released DP task models and does not use the original training data.

\paragraph{Vision tasks.}
For vision experiments, we follow the standard multi-task model-merging setting used in prior work~\citep{yang2024adamerging,cheng2025whoever}.
We consider eight image classification tasks: SUN397~\citep{xiao2016sun} for scene recognition, Cars~\citep{krause20133d} for fine-grained car classification, RESISC45~\citep{cheng2017remote} and EuroSAT~\citep{helber2019eurosat} for remote-sensing and land-cover recognition, SVHN~\citep{yuval2011reading} and MNIST~\citep{lecun1998mnist} for digit recognition, GTSRB~\citep{stallkamp2011german} for traffic-sign recognition, and DTD~\citep{cimpoi2014describing} for texture classification.
These tasks cover diverse visual domains and therefore provide a broad testbed for evaluating whether DP task models remain mergeable across heterogeneous classification problems.
For each dataset, we independently fine-tune a private task model from the same pretrained vision backbone under sample-level differential privacy.
The released DP task models are then merged into a single model without accessing the original training data.
We report classification accuracy on each task and use the average accuracy across all eight tasks as the main vision metric.

\paragraph{Language tasks.}
For language experiments, we use eight tasks from the GLUE benchmark~\citep{wang2018glue}: CoLA for linguistic acceptability, SST-2 for sentiment classification, MRPC and QQP for paraphrase detection, STS-B for semantic textual similarity, MNLI and RTE for natural language inference, and QNLI for question-answering natural language inference.
For each task, we fine-tune a separate private language model under sample-level differential privacy and merge the released DP task models without accessing the original GLUE training data.
Following the standard GLUE protocol, we report Matthew's correlation for CoLA, Spearman correlation for STS-B, F1/accuracy for MRPC and QQP, and accuracy for SST-2, MNLI, QNLI, and RTE.
All task scores are converted to a $0$--$100$ scale, and the average score over the eight tasks is used as the main language metric.

\subsection{Models}
\label{app:Models}

\paragraph{Vision models.}
For vision experiments, we use CLIP-pretrained Vision Transformer backbones~\citep{radford2021learning}, including ViT-B/32, ViT-B/16, and ViT-L/14.
ViT-B/32 and ViT-B/16 share the same base-size Transformer architecture with 12 layers and hidden dimension 768, but use different patch sizes of $32 \times 32$ and $16 \times 16$, respectively.
ViT-L/14 is a larger backbone with 24 layers, hidden dimension 1024, and $14 \times 14$ image patches.
For each backbone, all task-specific models are initialized from the same public CLIP checkpoint and independently fine-tuned on each private vision dataset under sample-level DP.
During merging, we merge the shared visual encoder parameters of the released DP task models.
Since the eight vision datasets have different label spaces, task-specific classification heads are kept separate for evaluation.
Thus, the merged vision model consists of one shared visual encoder and the corresponding task head for each dataset.

\paragraph{Language models.}
For language experiments, we use RoBERTa-Base and RoBERTa-Large~\citep{liu2019roberta}.
RoBERTa-Base has 12 Transformer layers with hidden dimension 768, while RoBERTa-Large has 24 Transformer layers with hidden dimension 1024. For each GLUE task, all task-specific models are initialized from the same public RoBERTa checkpoint and independently fine-tuned under sample-level DP.
During merging, we merge the shared RoBERTa encoder parameters of the released DP task models. Because GLUE tasks have different label spaces and output formats, task-specific prediction heads are kept separate for evaluation.
The merged language model, therefore, uses one shared encoder together with the corresponding task head for each evaluation task.

\subsection{Implementation Details}
\label{app:Implementation Details}
This section provides additional implementation details that are omitted from the main text due to space limits. All experiments are conducted on NVIDIA RTX 5090 GPUs.

\paragraph{Training schedules.}
All task-specific models are fine-tuned using sample-level differential privacy with gradient clipping $C=0.2$, batch size $16$, learning rate $1\times10^{-5}$, weight decay $1\times10^{-2}$, and $10$ training epochs/steps. We use AdamW as the optimizer with parameters $(\beta_1,\beta_2)=(0.9,0.999)$.  
For clarity, different task types follow the same DP fine-tuning defaults; any deviations from this default schedule are explicitly noted in the corresponding experiment description in the main text.

\paragraph{Merging hyperparameters.}
For all experiments, we use fixed merging hyperparameters for each merging method to ensure fair comparison:  
For Task Arithmetic, the scaling coefficient is $0.3$.  
For TIES-Merging, pruning density is $0.2$.  
For PCB Merging and WUDI-Merging, we use default configurations from the original papers unless otherwise specified.  

The same merging hyperparameters are used for both the standard DP fine-tuning baselines and DP-Merging. For DP-Merging specifically, the perturbation radius $\rho=0.05$ and the alignment coefficient $\lambda=0.05$ are used consistently across all tasks.

\begin{table*}[ht]
\centering
\caption{Multi-task accuracy (\%) on the 8-task vision benchmark with ViT-B/16, $\varepsilon=4$.}
\label{tab:performance_vitbase16}
\renewcommand\arraystretch{0.95}
\setlength{\tabcolsep}{5pt}
\resizebox{\linewidth}{!}{
\begin{tabular}{l|cccccccc|c}
\toprule
\textbf{Method}
& \textbf{SUN397}
& \textbf{Cars}
& \textbf{RESISC45}
& \textbf{EuroSAT}
& \textbf{SVHN}
& \textbf{GTSRB}
& \textbf{MNIST}
& \textbf{DTD}
& \textbf{Avg Acc} \\
\midrule
DP + Individual
& 59.8 & 68.2 & 88.1 & 86.9 & 86.2 & 89.4 & 90.7 & 70.6 & 80.0 \\
\midrule
DP + Weight Averaging
& 45.3 & 43.4 & 51.4 & 51.7 & 44.2 & 32.8 & 67.5 & 30.1 & 46.8 \\
\rowcolor{flatbg}
DP-Merging + Weight Averaging
& 49.4 & 48.4 & 57.1 & 56.2 & 49.4 & 39.4 & 75.1 & 38.6 & $52.1_{\textcolor{red}{+5.3}}$ \\

DP + RegMean
& 45.3 & 43.8 & 55.9 & 58.6 & 58.2 & 47.4 & 73.7 & 36.0 & 52.1 \\
\rowcolor{flatbg}
DP-Merging + RegMean
& 48.3 & 47.8 & 59.9 & 64.6 & 63.6 & 49.4 & 76.7 & 44.0 & $56.3_{\textcolor{red}{+4.2}}$ \\

DP + Task Arithmetic
& 45.2 & 44.9 & 56.7 & 68.4 & 70.5 & 49.7 & 77.3 & 40.4 & 57.9 \\
\rowcolor{flatbg}
DP-Merging + Task Arithmetic
& 49.1 & 48.5 & 59.4 & 72.6 & 78.1 & 56.5 & 79.1 & 46.3 & $61.4_{\textcolor{red}{+3.5}}$ \\

DP + TIES-Merging
& 39.8 & 38.7 & 40.7 & 59.5 & 66.2 & 52.1 & 78.3 & 44.2 & 52.3 \\
\rowcolor{flatbg}
DP-Merging + TIES-Merging
& 45.8 & 46.6 & 45.7 & 63.8 & 69.2 & 55.1 & 81.3 & 46.2 & $57.1_{\textcolor{red}{+4.8}}$ \\

DP + PCB Merging
& 45.7 & 43.9 & 46.5 & 57.2 & 61.7 & 50.3 & 77.0 & 37.1 & 52.6 \\
\rowcolor{flatbg}
DP-Merging + PCB Merging
& 49.4 & 49.0 & 51.2 & 61.3 & 65.1 & 55.8 & 81.1 & 38.6 & $56.4_{\textcolor{red}{+3.8}}$ \\

DP + WUDI-Merging
& 46.7 & 45.5 & 58.6 & 59.3 & 66.4 & 57.1 & 78.2 & 39.1 & 56.3 \\
\rowcolor{flatbg}
DP-Merging + WUDI-Merging
& 49.6 & 48.1 & 61.3 & 63.8 & 71.5 & 61.3 & 81.6 & 45.7 & $60.2_{\textcolor{red}{+3.9}}$ \\
\bottomrule
\end{tabular}
}
\end{table*}

\subsection{More Results on Vision Tasks}
\label{app:More Results on Vision Tasks}

Table~\ref{tab:performance_vitbase16} reports the multi-task accuracy (\%) of various training strategies on eight vision benchmarks. Each method is evaluated, including individual training (DP + Individual), simple weight averaging (DP + Weight Averaging), and various DP-Merging variants combined with regularization or task-specific aggregation strategies. Individual training achieves the highest average accuracy (80.0\%), but does not leverage knowledge sharing across tasks, whereas simple weight averaging results in a substantially lower average accuracy (46.8\%). Incorporating DP-Merging with different weighting or aggregation strategies significantly improves multi-task performance, with DP-Merging + Task Arithmetic reaching 61.4\%, and DP-Merging + WUDI-Merging achieving the highest average accuracy among merged approaches (60.2\%), demonstrating that careful parameter merging and task coordination can effectively enhance generalization in multi-task settings. The numbers in red indicate improvements relative to the corresponding non-merged baseline, highlighting the positive impact of merging strategies.

\noindent
\textbf{Loss Landscape Visualization.} 
To illustrate the effect of different DP fine-tuning and merging strategies on the geometry of merged models, we visualize the loss landscapes of the final merged ViT models under privacy budget $\varepsilon=4$. 
Figure~\ref{fig:naive_dp_loss} shows the landscapes of models merged using Task Arithmetic after Naive DP fine-tuning, while Figure~\ref{fig:dp_merging_loss} shows the landscapes of models merged using Task Arithmetic after DP-Merging. 
All visualized models are final merged models for three ViT variants (ViT-B/32, ViT-B/16, ViT-L/14). 
Comparing the two sets of landscapes, we observe that DP-Merging produces wider and smoother low-loss regions, indicating improved stability and geometric compatibility of the merged models.

\begin{figure*}[h]
    \centering
    \begin{subfigure}[b]{0.3\textwidth}
        \centering
        \includegraphics[width=\textwidth]{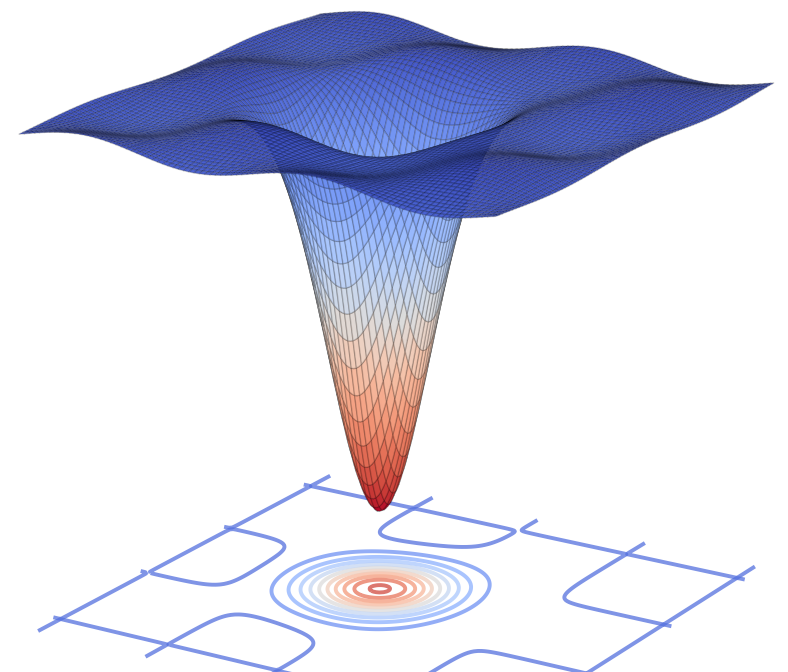}
        \caption{ViT-B/32}
        \label{fig:bar_1}
    \end{subfigure}
    \begin{subfigure}[b]{0.3\textwidth}
        \centering
        \includegraphics[width=\textwidth]{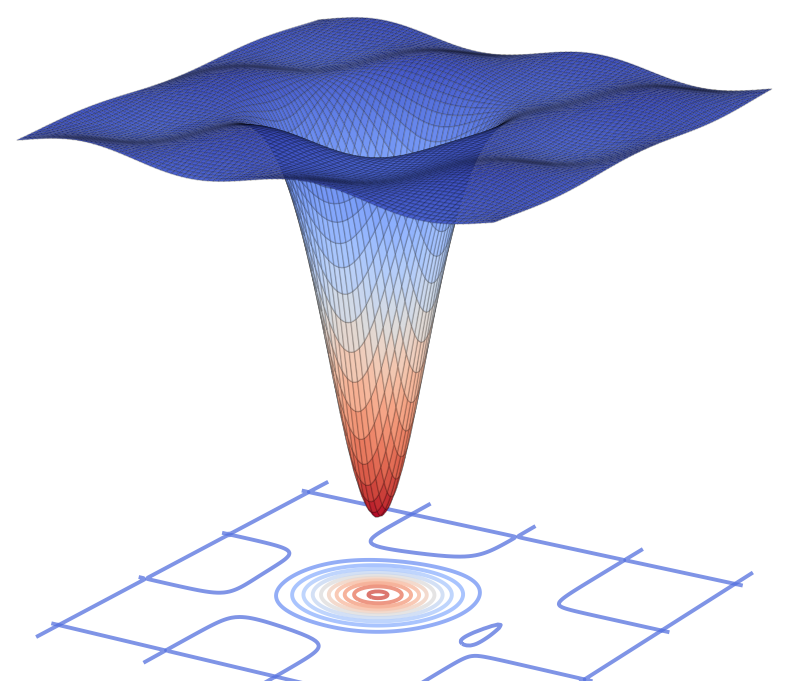}
        \caption{ViT-B/16}
        \label{fig:bar_2}
    \end{subfigure}
    \begin{subfigure}[b]{0.3\textwidth}
        \centering
        \includegraphics[width=\textwidth]{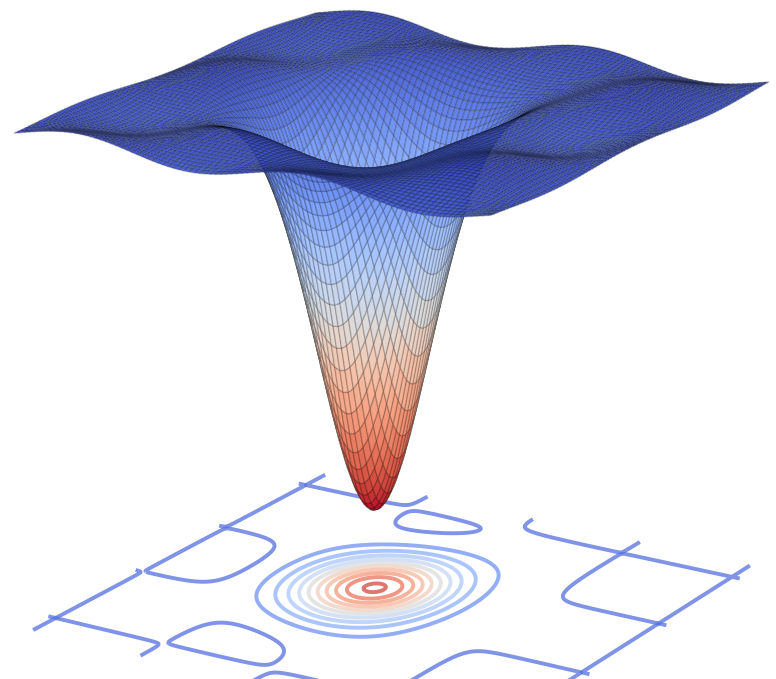}
        \caption{ViT-L/14}
        \label{fig:bar_3}
    \end{subfigure}
\caption{
Loss landscapes of ViT models that were first Naive DP fine-tuned under $\varepsilon=4$ and then merged using Task Arithmetic (ViT-B/32, ViT-B/16, ViT-L/14).
}
    \label{fig:naive_dp_loss}
    \vspace{-3mm}
\end{figure*}

\begin{figure*}[h]
    \centering
    \begin{subfigure}[b]{0.3\textwidth}
        \centering
        \includegraphics[width=\textwidth]{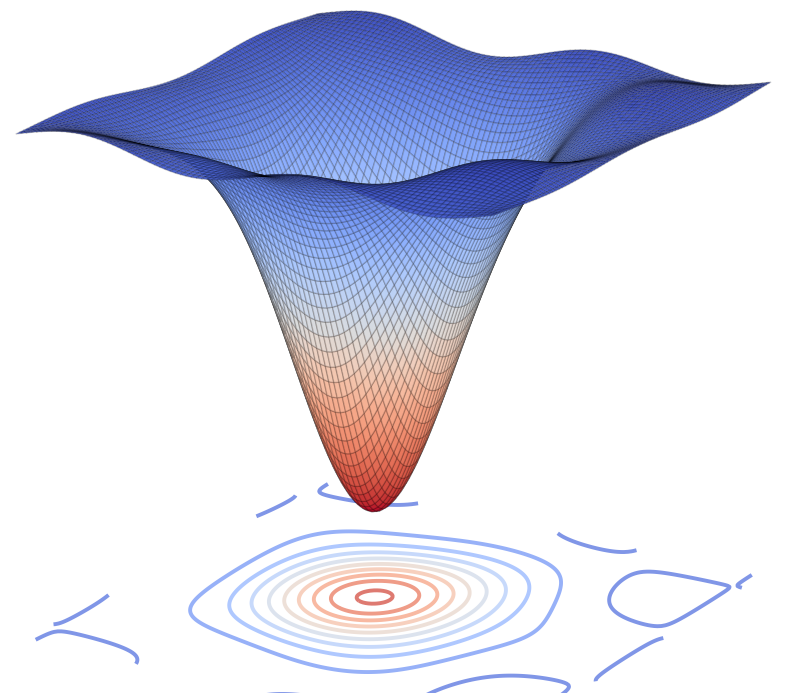}
        \caption{ViT-B/32}
        \label{fig:bar_1}
    \end{subfigure}
    \begin{subfigure}[b]{0.3\textwidth}
        \centering
        \includegraphics[width=\textwidth]{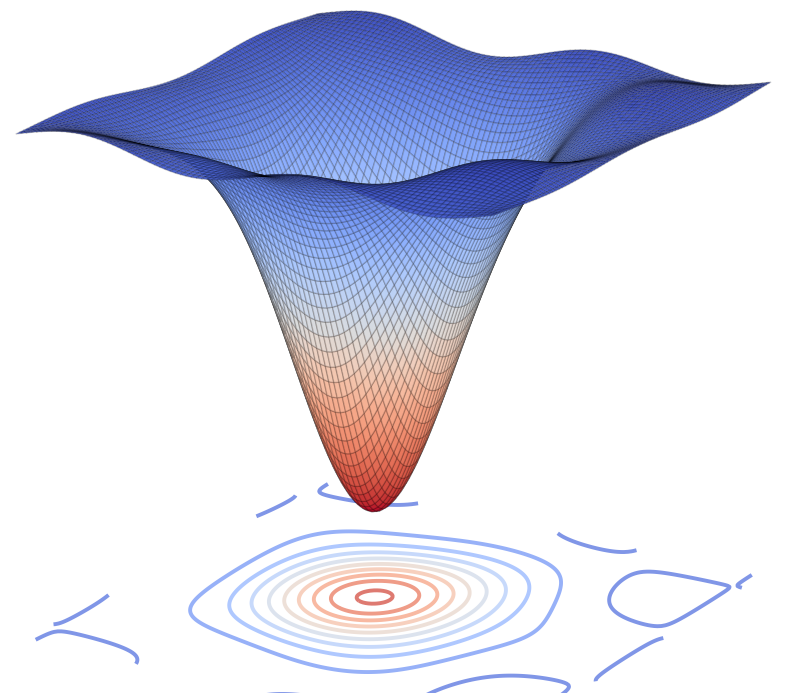}
        \caption{ViT-B/16}
        \label{fig:bar_2}
    \end{subfigure}
    \begin{subfigure}[b]{0.3\textwidth}
        \centering
        \includegraphics[width=\textwidth]{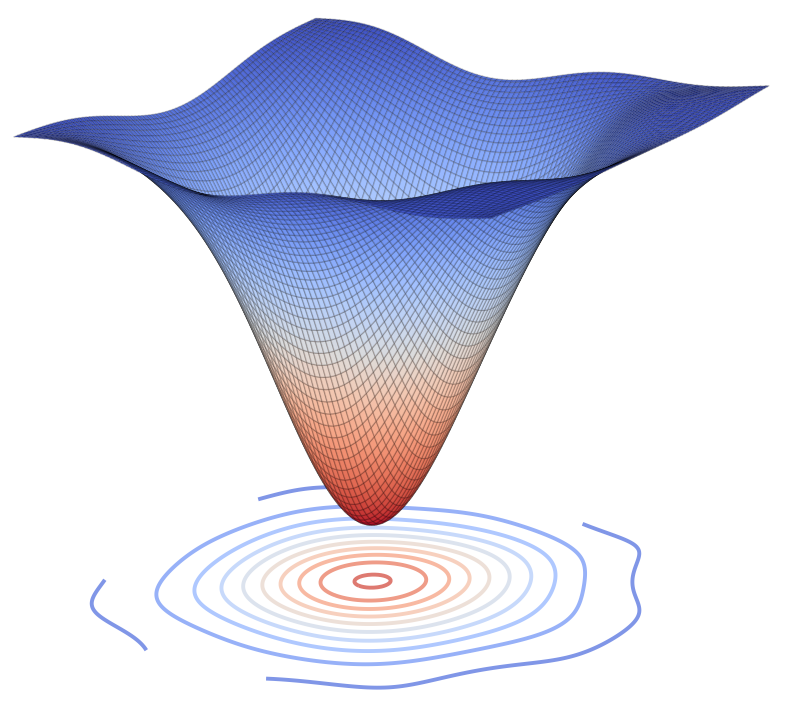}
        \caption{ViT-L/14}
        \label{fig:bar_3}
    \end{subfigure}
\caption{
Loss landscapes of the merged ViT models (ViT-B/32, ViT-B/16, and ViT-L/14) under $\varepsilon=4$, obtained using DP-Merging with Task Arithmetic. 
}
    \label{fig:dp_merging_loss}
    \vspace{-3mm}
\end{figure*}

\subsection{More Results on Language Tasks}
\label{app:More Results on Language Tasks}

Table~\ref{tab:performance_roberta_base} reports the multi-task performance of RoBERTa-Base models on the GLUE benchmark (CoLA, MNLI, MRPC, QNLI, QQP, RTE, SST2, STSB). Individual training (DP + Individual) achieves the highest average score (79.8), while simple weight averaging drops it to 72.3. DP-Merging strategies consistently improve results, with DP-Merging + Task Arithmetic reaching 73.9 and DP-Merging + WUDI-Merging achieving 76.5, demonstrating that parameter merging and task coordination enhance multi-task generalization. 

\begin{table*}[h]
\centering
\caption{Multi-task performance of RoBERTa-Base models on GLUE benchmark, $\varepsilon=4$.}
\label{tab:performance_roberta_base} 
\renewcommand\arraystretch{0.95}
\setlength{\tabcolsep}{8pt}
\resizebox{\linewidth}{!}{  
\begin{tabular}{l|cccccccc|c}
\toprule
\textbf{Method} & \textbf{CoLA} & \textbf{MNLI} & \textbf{MRPC} & \textbf{QNLI} & \textbf{QQP}  & \textbf{RTE}  & \textbf{SST2} & \textbf{STSB} & \textbf{Avg.} \\
\midrule
%\multicolumn{10}{c}{\emph{Non-merging Methods}} \\
DP + Individual
& 59.2 & 77.7 & 80.1 & 88.7 & 81.5 & 76.7 & 92.7 & 79.4 & 79.8 \\
\midrule
DP + Weight Averaging
& 55.3 & 73.4 & 67.4 & 78.7 & 78.2 & 68.8 & 90.5 & 66.1 & 72.3 \\
\rowcolor{flatbg}
DP-Merging + Weight Averaging
& 58.2 & 76.1 & 69.0 & 79.1 & 81.3 & 69.7 & 91.6 & 68.2 & $73.0_{\textcolor{red}{+0.7}}$ \\

DP + RegMean
& 55.3 & 73.5 & 65.6 & 78.6 & 78.1 & 67.4 & 90.7 & 62.0 & 70.4 \\
\rowcolor{flatbg}
DP-Merging + RegMean
& 58.6 & 78.2 & 68.4 & 79.0 & 80.6 & 69.2 & 91.6 & 63.8 & $73.7_{\textcolor{red}{+3.3}}$ \\

DP + Task Arithmetic
& 55.2 & 74.9 & 66.7 & 78.9 & 80.2 & 69.7 & 90.3 & 61.4 & 70.8 \\
\rowcolor{flatbg}
DP-Merging + Task Arithmetic
& 59.6 & 77.8 & 68.9 & 80.9 & 82.0 & 70.7 & 91.6 & 63.2 & $73.9_{\textcolor{red}{+3.1}}$ \\

DP + TIES-Merging
& 59.8 & 78.6 & 70.7 & 79.7 & 81.2 & 70.1 & 91.3 & 64.2 & 73.7 \\
\rowcolor{flatbg}
DP-Merging + TIES-Merging
& 61.0 & 80.6 & 72.4 & 80.9 & 82.1 & 71.2 & 92.6 & 65.8 & $76.3_{\textcolor{red}{+2.6}}$ \\

DP + PCB Merging
& 55.7 & 73.6 & 76.5 & 77.2 & 81.7 & 70.3 & 91.0 & 60.1 & 72.8 \\
\rowcolor{flatbg}
DP-Merging + PCB Merging
& 58.6 & 74.2 & 78.3 & 79.5 & 82.9 & 72.4 & 92.3 & 61.6 & $74.9_{\textcolor{red}{+2.1}}$ \\

DP + WUDI-Merging
& 56.7 & 75.5 & 78.5 & 79.3 & 82.4 & 73.1 & 91.2 & 59.0 & 73.6 \\
\rowcolor{flatbg}
DP-Merging + WUDI-Merging
& 59.6 & 76.9 & 81.0 & 82.1 & 83.1 & 74.6 & 92.8 & 60.5 & $76.5_{\textcolor{red}{+2.9}}$ \\
\bottomrule
\end{tabular}
}
\end{table*}

\section{Implementation of Theoretical Analysis}
\label{app:theoretical}

\subsection{Privacy guarantee of DP-Merging}
\label{app:privacy}

Each DP-Merging iteration uses two privatized clipped-gradient evaluations: one to construct the sharpness-aware perturbation and one to update the model at the perturbed point.
The perturbation itself is a deterministic function of the first private gradient
and therefore is post-processing. The reference-anchoring term depends only on
the current model and the public initialization, so it incurs no additional
privacy loss. Finally, the merging step only processes already released private
task models and is also post-processing.

\begin{theorem}[Privacy guarantee of DP-Merging]
\label{thm:privacy_budget}
Assume that each task dataset $\mathcal{D}_t$ is disjoint, and each private
example belongs to at most one task. For task $t$, suppose Algorithm~\ref{alg:dp_sam_merging}
runs for $K$ iterations with Poisson sampling rate $q$, clipping threshold $C$,
and Gaussian noise multiplier $\sigma$. Let
$\varepsilon_{\mathrm{pair}}(\alpha;q,\sigma)$ denote the order-$\alpha$ RDP cost
of one Poisson-subsampled paired Gaussian mechanism that releases the two noisy
clipped-gradient quantities used in one DP-Merging iteration. Then, for any
R\'enyi order $\alpha>1$, the released task model $w_t^{\mathrm{DP}}$ satisfies
\[
(\alpha,\, K\varepsilon_{\mathrm{pair}}(\alpha;q,\sigma))\text{-RDP}.
\]
Consequently, for any $\delta_t\in(0,1)$, $w_t^{\mathrm{DP}}$ satisfies
$(\varepsilon_t,\delta_t)$-DP with
\[
\varepsilon_t
=
\min_{\alpha>1}
\left\{
K\varepsilon_{\mathrm{pair}}(\alpha;q,\sigma)
+
\frac{\log(1/\delta_t)}{\alpha-1}
\right\}.
\]
Since the task datasets are disjoint, releasing all private task models
$\{w_t^{\mathrm{DP}}\}_{t=1}^{T}$ satisfies
\[
\left(
\max_{t\in[T]}\varepsilon_t,\,
\max_{t\in[T]}\delta_t
\right)\text{-DP}.
\]
The merged model
\[
w_{\mathrm{merge}}
=
\frac{1}{T}\sum_{t=1}^{T} w_t^{\mathrm{DP}}
\]
incurs no additional privacy loss by post-processing.
\end{theorem}

We provide a detailed proof of Theorem~\ref{thm:privacy_budget}. The analysis is
at the sample level and uses the standard add/remove neighboring relation: two
datasets are neighboring if they differ in the presence or absence of one
training example. We assume that each private example appears in at most one task
dataset $\mathcal{D}_t$. The replace-one neighboring relation can be handled by
doubling the clipping sensitivity, equivalently replacing $\sigma$ by $\sigma/2$
in the accountant.

For a per-example loss $\ell(w;z)$, define the clipped gradient
\[
\bar g(w;z)
=
\nabla \ell(w;z)
\cdot
\min\left\{
1,\,
\frac{C}{\|\nabla \ell(w;z)\|_2}
\right\}.
\]
Thus,
\[
\|\bar g(w;z)\|_2 \le C
\]
for all $w$ and $z$. Given a Poisson minibatch $B\subseteq \mathcal{D}_t$ sampled
with probability $q$, the private gradient oracle used by DP-Merging can be
written as
\[
\tilde g(w;B)
=
\frac{1}{s}
\left(
\sum_{z_i\in B}\bar g(w;z_i)
+
\xi
\right),
\qquad
\xi\sim \mathcal{N}(0,\sigma^2 C^2 I),
\]
where $s$ is a deterministic normalization factor. The value of $s$ does not
affect the privacy analysis because it scales both the sensitivity and the noise
by the same factor.

\paragraph{One-iteration mechanism.}
Fix a task $t$ and an iteration $k$. Conditioned on all previous private outputs,
the current parameter $w_{t,k}$ is fixed. In one iteration, DP-Merging first
computes
\[
\tilde g_{t,k}
=
\tilde g(w_{t,k};B_{t,k}),
\]
then constructs
\[
\epsilon_{t,k}
=
\rho_t
\frac{\tilde g_{t,k}}{\|\tilde g_{t,k}\|_2}.
\]
If $\|\tilde g_{t,k}\|_2=0$, we set $\epsilon_{t,k}=0$. In either case,
$\epsilon_{t,k}$ is a deterministic function of the first private output
$\tilde g_{t,k}$ and public hyperparameters. Therefore, constructing
$\epsilon_{t,k}$ is post-processing and does not increase privacy loss.

The second private gradient query is
\[
\tilde h_{t,k}
=
\tilde g(w_{t,k}+\epsilon_{t,k};B_{t,k}).
\]
Conditional on $\tilde g_{t,k}$, the perturbed point
$w_{t,k}+\epsilon_{t,k}$ is fixed. Hence, the second query is another Gaussian
mechanism applied to clipped per-example gradients at a fixed model parameter.
The two private gradient queries are adaptive, but adaptive composition is
allowed under RDP.

\paragraph{Full-batch paired Gaussian mechanism.}
We first ignore subsampling and analyze the paired mechanism that releases both
noisy clipped-gradient quantities on the same dataset. Let $D$ and $D'$ be
neighboring datasets that differ in one example. For the first query, the
difference between the two clipped gradient sums is bounded by
\[
\left\|
\sum_{z_i\in D}\bar g(w_{t,k};z_i)
-
\sum_{z_i\in D'}\bar g(w_{t,k};z_i)
\right\|_2
\le C.
\]
Similarly, after conditioning on the first private output, the perturbed model is
fixed, and the second query satisfies
\[
\left\|
\sum_{z_i\in D}\bar g(w_{t,k}+\epsilon_{t,k};z_i)
-
\sum_{z_i\in D'}\bar g(w_{t,k}+\epsilon_{t,k};z_i)
\right\|_2
\le C.
\]
Therefore, if we view the two released gradients as one concatenated vector, the
$\ell_2$ sensitivity of the paired query is bounded by
\[
\Delta_{\mathrm{pair}}
\le
\sqrt{C^2+C^2}
=
\sqrt{2}\,C.
\]
The paired mechanism adds independent Gaussian noise with covariance
$\sigma^2C^2 I$ to each query. Hence, for any R\'enyi order $\alpha>1$, the
full-batch paired Gaussian mechanism satisfies
\[
\varepsilon_{\mathrm{full}}(\alpha)
\le
\frac{\alpha \Delta_{\mathrm{pair}}^2}{2\sigma^2 C^2}
\le
\frac{\alpha}{\sigma^2}.
\]
Equivalently, this is the same as composing two Gaussian mechanisms, each with
RDP cost $\alpha/(2\sigma^2)$:
\[
\frac{\alpha}{2\sigma^2}
+
\frac{\alpha}{2\sigma^2}
=
\frac{\alpha}{\sigma^2}.
\]

\paragraph{Poisson subsampling.}
Algorithm~\ref{alg:dp_sam_merging} samples the minibatch $B_{t,k}$ using
Poisson sampling with rate $q$. Since the two gradient evaluations in one
iteration use the same minibatch, privacy amplification must be applied to the
joint paired mechanism rather than independently to the two queries.

For integer orders $\alpha\ge 2$, a valid RDP upper bound for one subsampled
paired Gaussian iteration is
\[
\varepsilon_{\mathrm{pair}}^{+}(\alpha;q,\sigma)
=
\frac{1}{\alpha-1}
\log
\left(
\sum_{j=0}^{\alpha}
\binom{\alpha}{j}
(1-q)^{\alpha-j}
q^j
\exp\left(
\frac{j(j-1)}{\sigma^2}
\right)
\right).
\]
The exponent differs from the standard single-query subsampled Gaussian
mechanism by a factor of two because the paired query has sensitivity
$\sqrt{2}C$ rather than $C$.

For the reverse neighboring direction, a simple valid bound is
\[
\varepsilon_{\mathrm{pair}}^{-}(\alpha;q,\sigma)
\le
-\log(1-q).
\]
Thus, one may take
\[
\varepsilon_{\mathrm{pair}}(\alpha;q,\sigma)
=
\max\left\{
\varepsilon_{\mathrm{pair}}^{+}(\alpha;q,\sigma),
\varepsilon_{\mathrm{pair}}^{-}(\alpha;q,\sigma)
\right\}.
\]
In practice, the same theorem also holds when
$\varepsilon_{\mathrm{pair}}(\alpha;q,\sigma)$ is computed by a tighter numerical
RDP accountant for the subsampled paired Gaussian mechanism.

\paragraph{Composition over iterations.}
For a fixed task $t$, DP-Merging runs for $K$ iterations. RDP composes
additively under adaptive composition. Therefore, after $K$ iterations, the
released private task model $w_t^{\mathrm{DP}}$ satisfies
\[
(\alpha,\varepsilon_t^{\mathrm{RDP}}(\alpha))\text{-RDP},
\qquad
\varepsilon_t^{\mathrm{RDP}}(\alpha)
\le
K\varepsilon_{\mathrm{pair}}(\alpha;q,\sigma).
\]
Using the standard conversion from RDP to approximate DP, for any
$\delta_t\in(0,1)$, $w_t^{\mathrm{DP}}$ is $(\varepsilon_t,\delta_t)$-DP with
\[
\varepsilon_t
=
\min_{\alpha>1}
\left\{
K\varepsilon_{\mathrm{pair}}(\alpha;q,\sigma)
+
\frac{\log(1/\delta_t)}{\alpha-1}
\right\}.
\]

\paragraph{No privacy cost from reference anchoring.}
The update rule of DP-Merging is
\[
w_{t,k+1}
=
w_{t,k}
-
\eta
\left(
\tilde h_{t,k}
+
2\lambda(w_{t,k}-w_0)
\right).
\]
The anchoring term $2\lambda(w_{t,k}-w_0)$ depends only on the current model
parameter $w_{t,k}$ and the public pretrained initialization $w_0$. It does not
directly query any private example. Therefore, after the private gradient
$\tilde h_{t,k}$ has been produced, adding the reference-anchoring term is a
deterministic transformation of already privatized quantities and public
information. By post-processing, it incurs no additional privacy loss.

\paragraph{Parallel composition over tasks.}
The $T$ task models are trained on disjoint datasets
$\{\mathcal{D}_t\}_{t=1}^{T}$. Since each private example belongs to at most one
task, changing one example affects the training procedure of at most one task.
Therefore, releasing all private task models
\[
\{w_t^{\mathrm{DP}}\}_{t=1}^{T}
\]
satisfies parallel composition. Hence, the collection of released task models is
\[
\left(
\max_{t\in[T]}\varepsilon_t,\,
\max_{t\in[T]}\delta_t
\right)\text{-DP}.
\]
When all tasks use the same privacy parameters, this reduces to
$(\varepsilon_t,\delta_t)$-DP.

\paragraph{Post-processing by model merging.}
The final merged model is computed as
\[
w_{\mathrm{merge}}
=
\frac{1}{T}
\sum_{t=1}^{T}w_t^{\mathrm{DP}}.
\]
This operation is a deterministic function of the released private task models
and does not access any private training example. Therefore, by the
post-processing property of differential privacy, releasing
$w_{\mathrm{merge}}$ incurs no additional privacy loss. This proves
Theorem~\ref{thm:privacy_budget}.

\subsection{Mergeability Analysis}
\label{app:mergeability}

Let $w_t:=w_t^{\mathrm{DP}}$ be the private task model returned by
Algorithm~\ref{alg:dp_sam_merging}. Define
\[
u_t:=w_t-w_0,
\qquad
\bar u:=\frac{1}{T}\sum_{s=1}^{T}u_s,
\qquad
w_{\mathrm{merge}}:=w_0+\bar u,
\]
and
\[
\Delta_t:=w_{\mathrm{merge}}-w_t=\bar u-u_t .
\]
The average merge gap is
\[
G_{\mathrm{merge}}
:=
\frac{1}{T}\sum_{t=1}^{T}
\bigl(\mathcal{L}_t(w_{\mathrm{merge}})-\mathcal{L}_t(w_t)\bigr).
\]

\begin{assumption}[Local third-order smoothness]
\label{assump:smooth}
For each task $t$, $\mathcal{L}_t$ is three-times differentiable along the segment
$\{w_t+\gamma\Delta_t:\gamma\in[0,1]\}$. Define
\[
\beta_t
:=
\sup_{\gamma\in[0,1]}
\lambda_{\max}\!\left(\nabla^2\mathcal{L}_t(w_t+\gamma\Delta_t)\right).
\]
The third-order Taylor remainder satisfies
\[
|R_t|
\le
\frac{M_t}{6}\|\Delta_t\|_2^3 .
\]
\end{assumption}

\begin{assumption}[Approximate stationarity]
\label{assump:stationary}
The returned private model satisfies
\[
\|\nabla \mathcal{L}_t(w_t)\|_2\le \varepsilon_t,
\]
where $\varepsilon_t$ captures the optimization error, clipping bias, and DP noise.
\end{assumption}

\begin{theorem}[Average merge-gap bound]
\label{thm:avg_merge_gap}
Under Assumptions~\ref{assump:smooth} and~\ref{assump:stationary},
\[
G_{\mathrm{merge}}
\le
\frac{1}{T}\sum_{t=1}^{T}
\left[
\varepsilon_t\|\Delta_t\|_2
+
\frac{\beta_t}{2}\|\Delta_t\|_2^2
+
\frac{M_t}{6}\|\Delta_t\|_2^3
\right].
\]
\end{theorem}

\begin{proof}
For each task $t$, apply Taylor's theorem to $\mathcal{L}_t$ at $w_t$ along the direction
$\Delta_t=w_{\mathrm{merge}}-w_t$:
\[
\mathcal{L}_t(w_{\mathrm{merge}})
=
\mathcal{L}_t(w_t)
+
\nabla \mathcal{L}_t(w_t)^\top \Delta_t
+
\frac{1}{2}
\Delta_t^\top
\nabla^2\mathcal{L}_t(w_t+\gamma_t\Delta_t)
\Delta_t
+
R_t
\]
for some $\gamma_t\in[0,1]$. Therefore,
\[
\mathcal{L}_t(w_{\mathrm{merge}})-\mathcal{L}_t(w_t)
\le
\|\nabla \mathcal{L}_t(w_t)\|_2\|\Delta_t\|_2
+
\frac{\beta_t}{2}\|\Delta_t\|_2^2
+
\frac{M_t}{6}\|\Delta_t\|_2^3 .
\]
Using Assumption~\ref{assump:stationary} gives
\[
\mathcal{L}_t(w_{\mathrm{merge}})-\mathcal{L}_t(w_t)
\le
\varepsilon_t\|\Delta_t\|_2
+
\frac{\beta_t}{2}\|\Delta_t\|_2^2
+
\frac{M_t}{6}\|\Delta_t\|_2^3 .
\]
Averaging over $t=1,\ldots,T$ completes the proof.
\end{proof}

\subsection{Reference Anchoring Controls Merging Displacement}
\label{app:anchor}

Theorem~\ref{thm:avg_merge_gap} shows that the merge gap depends on the displacement
$\|\Delta_t\|_2$. We now show that the reference anchor in DP-Merging controls this quantity.

Define the robust loss
\[
L_{t,\rho_t}(w)
:=
\max_{\|\epsilon\|_2\le \rho_t}
\mathcal{L}_t(w+\epsilon),
\]
and the anchored robust objective
\[
\Phi_t(w)
:=
L_{t,\rho_t}(w)+\lambda\|w-w_0\|_2^2 .
\]

\begin{assumption}[Approximate stationarity of the anchored robust objective]
\label{assump:robust_stationary}
For each task $t$, the returned model satisfies
\[
\|\nabla L_{t,\rho_t}(w_t)\|_2\le G_t,
\qquad
\|\nabla\Phi_t(w_t)\|_2\le \zeta_t .
\]
\end{assumption}

\begin{lemma}[Anchor-induced task-vector bound]
\label{lem:anchor_bound}
Under Assumption~\ref{assump:robust_stationary},
\[
\|w_t-w_0\|_2
\le
\frac{G_t+\zeta_t}{2\lambda}.
\]
Consequently, if
\[
A_t:=\frac{G_t+\zeta_t}{2\lambda},
\qquad
\bar A:=\frac{1}{T}\sum_{s=1}^{T}A_s,
\]
then
\[
\|\Delta_t\|_2
\le
A_t+\bar A .
\]
\end{lemma}

\begin{proof}
By definition,
\[
\nabla\Phi_t(w_t)
=
\nabla L_{t,\rho_t}(w_t)
+
2\lambda(w_t-w_0).
\]
Thus,
\[
2\lambda\|w_t-w_0\|_2
\le
\|\nabla\Phi_t(w_t)\|_2
+
\|\nabla L_{t,\rho_t}(w_t)\|_2
\le
\zeta_t+G_t,
\]
which proves the task-vector bound. Next,
\[
\begin{aligned}
\|\Delta_t\|_2
&=
\|\bar u-u_t\|_2                                      \\
&=
\left\|
\frac{1}{T}\sum_{s=1}^{T}u_s-u_t
\right\|_2                                             \\
&\le
\frac{1}{T}\sum_{s=1}^{T}\|u_s\|_2+\|u_t\|_2             \\
&\le
\bar A+A_t .
\end{aligned}
\]
This completes the proof.
\end{proof}

\begin{corollary}[Explicit DP-Merging merge-gap bound]
\label{cor:explicit_merge_gap}
Under Assumptions~\ref{assump:smooth}, \ref{assump:stationary}, and
\ref{assump:robust_stationary},
\[
G_{\mathrm{merge}}
\le
\frac{1}{T}\sum_{t=1}^{T}
\left[
\varepsilon_t(A_t+\bar A)
+
\frac{\beta_t}{2}(A_t+\bar A)^2
+
\frac{M_t}{6}(A_t+\bar A)^3
\right].
\]
If $\varepsilon_t\le\varepsilon$, $\beta_t\le\beta$, $M_t\le M$,
$G_t\le G$, and $\zeta_t\le\zeta$ for all $t$, then
\[
G_{\mathrm{merge}}
\le
\varepsilon\frac{G+\zeta}{\lambda}
+
\frac{\beta}{2}
\left(\frac{G+\zeta}{\lambda}\right)^2
+
\frac{M}{6}
\left(\frac{G+\zeta}{\lambda}\right)^3 .
\]
\end{corollary}

\begin{proof}
The first bound follows by substituting Lemma~\ref{lem:anchor_bound} into
Theorem~\ref{thm:avg_merge_gap}. For the uniform bound, note that
\[
A_t+\bar A
\le
\frac{G+\zeta}{2\lambda}
+
\frac{G+\zeta}{2\lambda}
=
\frac{G+\zeta}{\lambda}.
\]
Substituting this inequality into the first bound gives the result.
\end{proof}

\subsection{Connection to Joint-Task Loss Linearity}
\label{app:jtl}

We further connect the above mergeability analysis to joint-task loss linearity, following the same
style of Hessian-based arguments commonly used in sharpness-aware model-merging analyses.

For two tasks $s$ and $t$, define the joint-task loss
\[
L_{\mathrm{JTL}}(w;\mathcal D_s\cup\mathcal D_t)
:=
L_s(w)+\mathcal{L}_t(w).
\]
For $\alpha\in[0,1]$, define the joint-task loss linearity gap
\[
\delta_{s,t}(\alpha)
:=
L_{\mathrm{JTL}}(\alpha w_s+(1-\alpha)w_t)
-
\alpha L_{\mathrm{JTL}}(w_s)
-
(1-\alpha)L_{\mathrm{JTL}}(w_t).
\]
A smaller $|\delta_{s,t}(\alpha)|$ means that the interpolation between two task models is closer
to being linear on the joint-task loss landscape.

\begin{theorem}[Flatness and anchoring imply joint-task loss linearity]
\label{thm:jtl}
Assume $L_s$ and $\mathcal{L}_t$ are locally third-order smooth around $w_s$ and $w_t$, respectively. Let
\[
\lambda_s:=\lambda_{\max}(\nabla^2L_s(w_s)),
\qquad
\lambda_t:=\lambda_{\max}(\nabla^2\mathcal{L}_t(w_t)).
\]
Then
\[
|\delta_{s,t}(\alpha)|
\le
\frac{1}{2}\alpha(1-\alpha)(\lambda_s+\lambda_t)\|w_s-w_t\|_2^2
+
|R_{s,t}|,
\]
where $R_{s,t}$ collects the third-order Taylor remainders. Moreover, under
Lemma~\ref{lem:anchor_bound},
\[
\|w_s-w_t\|_2
\le
A_s+A_t,
\]
and therefore
\[
|\delta_{s,t}(\alpha)|
\le
\frac{1}{2}\alpha(1-\alpha)(\lambda_s+\lambda_t)(A_s+A_t)^2
+
|R_{s,t}|.
\]
\end{theorem}

\begin{proof}
Let $ v:=w_t-w_s.$
First expand $L_s(\alpha w_s+(1-\alpha)w_t)$ around $w_s$. Since
\[
\alpha w_s+(1-\alpha)w_t
=
w_s+(1-\alpha)v,
\]
we have
\[
L_s(w_s+(1-\alpha)v)
=
L_s(w_s)
+
(1-\alpha)\nabla L_s(w_s)^\top v
+
\frac{1}{2}(1-\alpha)^2v^\top\nabla^2L_s(w_s)v
+
R_s .
\]
Similarly,
\[
L_s(w_t)
=
L_s(w_s)
+
\nabla L_s(w_s)^\top v
+
\frac{1}{2}v^\top\nabla^2L_s(w_s)v
+
R_s' .
\]
Substituting these two expansions into
\[
\delta_s
:=
L_s(\alpha w_s+(1-\alpha)w_t)
-
\alpha L_s(w_s)
-
(1-\alpha)L_s(w_t)
\]
gives
\[
\delta_s
=
-\frac{1}{2}\alpha(1-\alpha)
v^\top\nabla^2L_s(w_s)v
+
R_s-(1-\alpha)R_s' .
\]
Repeating the same argument for $\mathcal{L}_t$ around $w_t$ gives
\[
\delta_t
=
-\frac{1}{2}\alpha(1-\alpha)
v^\top\nabla^2\mathcal{L}_t(w_t)v
+
R_t-\alpha R_t' .
\]
Therefore,
\[
\delta_{s,t}(\alpha)
=
-\frac{1}{2}\alpha(1-\alpha)
v^\top
\left(
\nabla^2L_s(w_s)+\nabla^2\mathcal{L}_t(w_t)
\right)
v
+
R_{s,t},
\]
where
\[
R_{s,t}:=
R_s-(1-\alpha)R_s'+R_t-\alpha R_t' .
\]
Using
\[
v^\top\nabla^2L_s(w_s)v
\le
\lambda_s\|v\|_2^2,
\qquad
v^\top\nabla^2\mathcal{L}_t(w_t)v
\le
\lambda_t\|v\|_2^2,
\]
we obtain
\[
|\delta_{s,t}(\alpha)|
\le
\frac{1}{2}\alpha(1-\alpha)(\lambda_s+\lambda_t)\|v\|_2^2
+
|R_{s,t}|.
\]
Finally,
\[
\|w_s-w_t\|_2
=
\|(w_s-w_0)-(w_t-w_0)\|_2
\le
\|w_s-w_0\|_2+\|w_t-w_0\|_2
\le
A_s+A_t,
\]
where the last inequality follows from Lemma~\ref{lem:anchor_bound}. This proves the theorem.
\end{proof}

\paragraph{Interpretation.}
Theorem~\ref{thm:jtl} shows that DP-Merging improves joint-task loss linearity through the same
two geometric mechanisms that improve mergeability. The sharpness-aware component reduces the
dominant Hessian eigenvalues $\lambda_s$ and $\lambda_t$, while the reference anchor reduces the
distance $\|w_s-w_t\|_2$ between task models. Together, they reduce the joint-task loss linearity
gap, which corresponds to weaker parameter interference during model merging.

\section{Limitations and Future Work}
\label{app:Limitations and Future Work}

While our DP-Merging and multi-task strategies demonstrate consistent improvements across vision and NLP benchmarks, several limitations remain. First, the effectiveness of parameter merging depends on task-specific compatibility, and its performance on highly heterogeneous tasks or models with substantially different architectures is not fully explored. Second, the experiments are limited to standard benchmarks and mid- to large-scale models (ViT-L/14, RoBERTa-Large), leaving the applicability to larger models or more complex, real-world datasets untested. Third, while average scores demonstrate clear overall gains, individual task performance can vary, with some tasks showing only marginal improvements or slight declines, highlighting the need for more task-specific analysis.

Future work includes developing adaptive merging strategies that account for task similarity and parameter sensitivity, extending the approach to larger and more diverse model architectures, and exploring dynamic task balancing and uncertainty-aware weighting mechanisms during training to improve stability and consistency across tasks.

\end{document}